\documentclass[sigconf]{aamas}

\usepackage{balance} %

\usepackage{latexsym}
\usepackage{amsmath}
\usepackage{graphicx}
\usepackage{color}
\usepackage{mathtools}
\usepackage[inline]{enumitem}

\usepackage{tikz}
\usepackage{titlesec}

\usepackage{orcidlink}

\doi{GOIS4183}

\makeatletter
\gdef\@copyrightpermission{
  \begin{minipage}{0.2\columnwidth}
   \href{https://creativecommons.org/licenses/by/4.0/}{\includegraphics[width=0.90\textwidth]{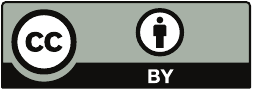}}
  \end{minipage}\hfill
  \begin{minipage}{0.8\columnwidth}
   \href{https://creativecommons.org/licenses/by/4.0/}{This work is licensed under a Creative Commons Attribution International 4.0 License.}
  \end{minipage}
  \vspace{5pt}
}
\makeatother

\setcopyright{ifaamas}
\acmConference[AAMAS '26]{Proc.\@ of the 25th International Conference
on Autonomous Agents and Multiagent Systems (AAMAS 2026)}{May 25 -- 29, 2026}
{Paphos, Cyprus}{C.~Amato, L.~Dennis, V.~Mascardi, J.~Thangarajah (eds.)}
\copyrightyear{2026}
\acmYear{2026}
\acmDOI{}
\acmPrice{}
\acmISBN{}

\acmSubmissionID{699}

\title[Lifted Forward Planning]{Lifted Forward Planning in Relational Factored Markov Decision Processes with Concurrent Actions}

\author{Florian Andreas Marwitz \orcidlink{0000-0002-9683-5250}}
\affiliation{
  \institution{University of Hamburg}
  \city{Hamburg}
  \country{Germany}}
\email{florian.marwitz@uni-hamburg.de}

\author{Tanya Braun \orcidlink{0000-0003-0282-4284}}
\affiliation{
  \institution{University of Münster}
  \city{Münster}
  \country{Germany}}
\email{tanya.braun@uni-muenster.de}

\author{Ralf Möller \orcidlink{0000-0002-1174-3323}}
\affiliation{
  \institution{University of Hamburg}
  \city{Hamburg}
  \country{Germany}}
\email{ralf.moeller@uni-hamburg.de}

\author{Marcel Gehrke \orcidlink{0000-0001-9056-7673}}
\affiliation{
  \institution{University of Hamburg}
  \city{Hamburg}
  \country{Germany}}
\email{marcel.gehrke@uni-hamburg.de}

\begin{abstract}

When allowing concurrent actions in Markov Decision Processes, whose state and action spaces grow exponentially in the number of objects, computing a policy becomes highly inefficient, as it requires enumerating the joint of the two spaces.
For the case of indistinguishable objects, we present a first-order representation to tackle the exponential blow-up in the action and state spaces.
We propose Foreplan, an efficient relational forward planner, which uses the first-order representation allowing to compute policies in space and time polynomially in the number of objects.
Thus, Foreplan significantly increases the number of planning problems solvable in an exact manner in reasonable time, which we underscore with a theoretical analysis.
To speed up computations even further, we also introduce an approximate version of Foreplan, including guarantees on the error.
Further, we provide an empirical evaluation of both Foreplan versions, demonstrating a speedup of several orders of magnitude.
For the approximate version of Foreplan, we also empirically show that the induced error is often negligible.
\end{abstract}


\newcommand{\BibTeX}{\rm B\kern-.05em{\sc i\kern-.025em b}\kern-.08em\TeX}

\newcommand{\customcite}[1]{\citet{#1}}

\DeclarePairedDelimiter\abs{\lvert}{\rvert}%
\DeclarePairedDelimiter\norm{\lVert}{\rVert}%
\makeatletter
\let\oldabs\abs
\def\abs{\@ifstar{\oldabs}{\oldabs*}}
\let\oldnorm\norm
\def\norm{\@ifstar{\oldnorm}{\oldnorm*}}
\makeatother

\usetikzlibrary{positioning}
\usetikzlibrary{arrows}
\usetikzlibrary{tikzmark}
\usetikzlibrary{calc}
\usetikzlibrary{shapes.geometric}

\tikzset{
node/.style={circle, draw=black},
block/.style={draw=black}
}
\tikzset{every picture/.style={line width=0.6pt}}

\usetikzlibrary{arrows.meta,positioning,shapes,fit,bending,quotes}

\tikzset{
  diedge/.style = {
   semithick, ->, >={[round,sep,bend]Stealth}
  }
}

\tikzset{
  biedge/.style = {
   semithick, <->, >={[round,sep,bend]Stealth}
  }
}

\newcommand\pfs[7]{
	\node[pf, #1=of #2, xshift=-1mm, yshift=1mm](#5) {};
	\node[pf, #1=of #2, label={[label distance=1mm]#3:{#4}}](#6) {};
	\node[pf, #1=of #2, xshift=1mm, yshift=-1mm](#7) {};
}

\newcommand\pfsat[6]{
    \node[pf, xshift=-1mm, yshift=1mm](#4) at (#1) {};
	\node[pf, label={[label distance=1mm]#2:{#3}}](#5) at (#1) {};
	\node[pf, xshift=1mm, yshift=-1mm](#6) at (#5) {};
}

\newtheorem*{theorem*}{Theorem}
\newtheorem*{example*}{Example}

\begin{document}

\pagestyle{fancy}
\fancyhead{}

\maketitle

\section{Introduction}

To compute a policy for a Markov decision processes (MDPs), combinations of the state and action spaces have to be enumerated to find the best possible action for each state.
With an increasing number of (indistinguishable) objects, the state space grows exponentially, leading to drastically more combinations to be enumerated.
In case the actions also depend on the states or objects, the action space grows as well, leading to even more combinations to enumerate.
Further, if one allows concurrent actions, all possible combinations of the actions of the already enlarged action space have to be accounted for during the enumeration, yielding an exponentially-sized action space. %
Therefore, given the number of combinations that have to be enumerated, concurrency is rarely modelled.
However, consider the following example: 
A small town is haunted by an epidemic. 
To fight the epidemic, the town's mayor can impose travel bans on the town's citizens.
Certainly, the mayor can confine all citizens to their homes, stopping the epidemic. 
However, the citizens' overall welfare is important as well. 
Therefore, the mayor is interested in the \emph{best decision} of imposing travel bans (concurrent actions) w.r.t.\ confining the epidemic, while keeping the citizens' welfare above a certain threshold.
Computing this problem on a propositional level, with every citizen explicitly represented, such as with current MDP approaches, blows up the state and action space, as it requires exponential enumeration of all subsets of the population.
However, there are groups of citizens behaving identically w.r.t.\ getting sick (and well again) as well as regarding their welfare if a travel ban is imposed, i.e., these citizens are indistinguishable for the mayor.
Within these groups, it does not matter on \emph{which exact} citizen the mayor imposes a travel ban, only on \emph{how many} she imposes a travel ban is important.
Additionally, many computations over subsets of indistinguishable citizens are redundant.
Thus, we propose to drastically reduce the search and action space using a first-order representation, grouping indistinguishable citizens, and use this representation to plan group-wise by reasoning about a representative and then projecting the result to the whole group.
This paper is an extended version of our paper accepted at AAMAS 2026~\cite{aamas} with appendix including extended examples and with all (complete) proofs.

\paragraph{Contribution}
First, to be able to represent numerous indistinguishable objects and actions, we use probabilistic relational models as a representation in (factored) MDPs (fMDPs), which yields \emph{relational factored MDPs} (rfMDPs).
Second, as our key contribution, we propose \emph{Foreplan} to carry out exact planning in rfMDPs.
Foreplan uses our first-order representation to keep the action and state space only polynomial in the number of objects, which we prove by using our so-called \emph{relational cost graph}.
Foreplan remains exponential in the number of cliques $c$ and size $w$ of the largest clique in the relational cost graph, but both parameters are mostly small and independent of the domain sizes.
Therefore, Foreplan is an efficient planner w.r.t the number of objects.
We show a speedup of at least four orders of magnitude is achievable.
Last, following the approximation ideas of \customcite{approx_factored_mdps}, we can reduce the runtime even further.
We propose Approximate Foreplan, whose runtime is polynomial in $c$ unlike Foreplan.
For Approximate Foreplan, we show a speedup of at least four orders of magnitude.

\paragraph{Related Work}
\customcite{mdp} introduces MDPs, which \customcite{factored_mdp} extend to fMDPs by factorizing the transition function.
Factorizing also the value function, \customcite{approx_factored_mdps} provide two approximate algorithms for solving planning in fMDPs.
\customcite{dean2013model} cluster the state space of fMDPs to reduce the state space even further. 
\customcite{givan2003equivalence} group equivalent states based on the notion of bisimulation. 
Both approaches lack the ability to handle concurrent actions efficiently.
MDPs can be generalized to partially observable MDPs, in which the agent is uncertain in which state the environment currently is~\cite{kaelbling1998planning}.
\customcite{sanner2010symbolic} add the first-order perspective to partially observable MDPs, but do not consider concurrent actions.
\customcite{bernstein2002complexity} extend partially observable MDPs to have a set of decentralized agents.
\customcite{braun2022lifting} group indistinguishable agents.
This is similar to our approach, in which we handle sets of indistinguishable state variables.
However, \customcite{braun2022lifting} only solve the problem inefficiently via brute-force.

The idea of lifting is to carry out computations over representatives for groups of indistinguishable random variables~\cite{lve,kersting2012lifted,raedt2016statistical}. %
There are online decision making approaches adding action and utility nodes to this representation~\cite{apsel2011extended,10.1007/978-3-030-12738-1_10,10.1007/978-3-030-18305-9_33}, here, we focus on offline planning.
To carry out even more lifted computations, \customcite{crv} extends lifted probabilistic inference by a generalized counting framework, which we extend later on.
Using a first-order representation for states, actions and objects, \customcite{boutilier2001symbolic} exploit the relational structure for planning using MDPs, still without concurrent actions.
To specify factored transition models in first-order MDPs, \customcite{sanner2007approximate} introduce factored first-order MDPs using a backward search, and give only an approximate algorithm without providing error bounds. %
In contrast, we propose an exact algorithm by applying lifting.
Later, we also introduce an approximate version with error estimation.
Moreover, we prove on which models our algorithm runs efficiently.
For a survey on planning using first-order representations, we refer to \customcite{correa2024lifted}.

\paragraph{Structure}
The remainder of this paper is structured as follows:
First, we present preliminaries for (factored) MDPs, including lifting. %
Then, we introduce rfMDPs to group indistinguishable objects.
Afterwards, we propose Foreplan, which exploits these indistinguishable objects using a compact state representation, to efficiently support decision making with concurrent actions.
Further, we provide a theoretical analysis and empirical evaluation of Foreplan.

\section{Preliminaries} \label{sec:preliminaries}
In this section, we lay the foundation for rfMDPs.
We first recap (f)MDPs, which model probabilistic state changes occurring due to agents performing actions.
Furthermore, states have a reward assigned, and the task is to compute the optimal policy, i.e., which action to perform in which state.
Second, we recap lifting in probabilistic graphical models.

\subsection{(Factored) Markov Decision Processes}
In this subsection, we first define MDPs and specialize them to fMDPs by factoring the transition function $T$.

\begin{definition}[Markov Decision Process] \label{definition:mdp}
    A \emph{Markov Decision Process} is a tuple $(\mathbf{S},\mathbf{A},T,R)$ with a finite set of states $\mathbf{S}$, a finite set of actions $\mathbf{A}$, a reward function $R: \mathbf{S} \mapsto \mathbb{R}$ and a transition function $T: \mathbf{S} \times \mathbf{A} \times \mathbf{S} \to [0,1]$.
    The value $T(s,a,s')$ is the probability $P(s' \mid s,a)$ of transitioning from state $s \in \mathbf{S}$ to state $s' \in \mathbf{S}$ if action $a \in \mathbf{A}$ is performed.
\end{definition}

The rewards are additive, possibly discounted by a factor $\gamma \in [0,1)$.
An MDP is fully observable and has the first order Markov property, i.e., the probability of the next state only depends on the current state and action.
Let us have a look at a simple, yet incomplete example of an MDP.

\begin{example} \label{ex:sick}
Suppose we have the states \emph{healthy} and \emph{sick}. Our agent has two possible actions: \emph{Travelling} or \emph{staying at home}.
When the agent is in state sick and travels, she stays sick with a probability of $0.9$.
The agent obtains a reward of $-1$ if sick and $1$ if healthy.
\end{example}

Planning in MDPs refers to calculating an optimal policy, which is a mapping from each state to an action to perform for the agent. %
To compute such a policy, we first define the \emph{utility} of a state:
\begin{definition}[Bellman Equation~\cite{mdp}] \label{def:bellman_equation}
The utility of a state $s$ is given by
\begin{equation}
    U(s) = R(s) + \gamma \max_{a \in \mathbf{A}} \sum_{s' \in \mathbf{S}} P(s' \mid s,a) \cdot U(s').
\end{equation}
\end{definition}

To find the utility of a state algorithmically, we determine a \emph{value} function $V$ satisfying the Bellman equation.
The value function induces a policy by selecting the action that yields the maximum expected value.
For computing a value function, we can use a linear programming formulation~\cite{mdp_lp_orig,mdp_lp}:
\begin{equation} \label{eq:lp}
\begin{array}{ll@{}}
\text{Variables:} & V(s) \quad \forall s \in \mathbf{S} \ ;\\
\text{Minimize:} & \displaystyle\sum_{s \in \mathbf{S}} \alpha(s) V(s) \ ;\\
\text{Subject to:} & \forall s \in \mathbf{S},\ \forall a \in \mathbf{A}: \\
& V(s) \geq R(s) + \gamma \displaystyle\sum_{s' \in \mathbf{S}} P(s' \mid s,a) \cdot V(s'), \\
\end{array}
\end{equation}
where the coefficients $\alpha(s)$ are arbitrary positive numbers, e.g., a uniform distribution over all states~\cite{mdp_lp}.
Planning in MDPs can be solved in polynomial time w.r.t to the state space size~\cite{papadimitriou1987complexity}. But what if the state space becomes very large, e.g., exponential in the number of objects?
For retaining an efficient transition model, fMDPs make use of state variables for the objects.
The state space is then spanned by the state variables. %
For simplicity, all state variables are Boolean, but can be easily extended to non-Boolean.
\begin{definition}[Factored MDP]
A \emph{factored MDP} is a tuple $(\mathbf{S}$, $\mathbf{A}$, $T$, $R)$, where the state space is the cross product $\mathbf{S} = \{0,1\}^m$ of the sub-spaces corresponding´to the state variables $S_1, \dots, S_m$ with Boolean range. The transition function $T$ is, for each action, factored:

\begin{equation}
P(S' \mid S,a) = \prod_{i=1}^m P(S_i' \mid Pa(S_i'),a),
\end{equation}
where $Pa(S_i') \subseteq \mathbf{S}$ denotes the set of \emph{parents} of $S_i'$ in the previous state, $S$ the old state, $S'$ the new state and $S_i'$ the $i$-th state variables in the respective state. For a given state $s \in \mathbf{S}$, we denote with $s_i$ the assignment of state variable $S_i$ in state $s$.
\end{definition}

For large $m$, the state space explodes.
But when these objects are indistinguishable, we can use parameterized graphical models to encode state spaces and transition models in a compact way.

\subsection{Parameterized Graphical Models}
Having \emph{indistinguishable} random variables leads to redundant computations.
We can tackle redundant computations by parameterizing our probabilistic model and using representatives of groups of indistinguishable variables, so that inference in the probabilistic model becomes tractable w.r.t. domain sizes using representatives during calculations~\cite{niepert2014tractability}.
We first define \emph{parameterized random variables} (PRVs) to group same behaving variables:
\begin{definition}[Parameterized Random Variable~\cite{parfactor_definitions}]
	Let $\mathbf{W}$ be a set of random variable names, $\mathbf{L}$ a set of logical variable (logvar) names, and $\mathbf{D}$ a set of constants (universe).
	All sets are finite.
	Each logvar $L$ has a domain $\mathcal{D}(L) \subseteq \mathbf{D}$.
	A \emph{constraint} $C$ is a tuple $(\mathcal{X}, C_{\mathcal{X}})$ of a sequence of logvars $\mathcal{X} = (X_1, \dots, X_n)$ and a set $C_{\mathcal{X}} \subseteq \times_{i = 1}^n\mathcal{D}(X_i)$.
	The symbol $\top$ for $C$ marks that no restrictions apply, i.e., $C_{\mathcal{X}} = \times_{i = 1}^n\mathcal{D}(X_i)$.
	A \emph{PRV} $B(L_1, \dots, L_n), n \geq 0,$ is a syntactical construct of a random variable $B \in \mathbf{W}$ possibly combined with logvars $L_1, \dots, L_n \in \mathbf{L}$. %
	If $n = 0$, the PRV is parameterless and constitutes a propositional random variable (RV).
	The term $\mathcal{R}(B)$ denotes the possible values (range) of a PRV $B$. 
	An \emph{event} $B = b$ denotes the occurrence of PRV $B$ with range value $b \in \mathcal{R}(B)$.
\end{definition}
We create PRVs for the health status of each person in the town:
\begin{example} \label{example:prvs}
Let $\mathbf{W}=\{Sick,Epidemic\}$, $\mathbf{L}=\{M\}$, $\mathcal{D}(M)=\mathbf{D}=\{a,b,c,d,e,f,g,h\}$ with Boolean-valued PRVs $Sick(M)$ and $Epidemic$.
$Epidemic$ is a parameterless PRV.
\end{example}

To relate PRVs, we use \emph{parameterized factors}:
\begin{definition}[Parfactor model~\cite{parfactor_definitions}]
    Let $\Phi$ be a set of factor names.
	We denote a \emph{parameterized factor} (\emph{parfactor}) $g$ by $\phi(\mathcal{B})_{| C}$ with $\mathcal{B} = (B_1, \dots, B_n)$ a sequence of PRVs, $\phi : \times_{i = 1}^n \mathcal{R}(B_i) \to \mathbb{R}^+$ a \emph{potential function} with name $\phi \in \Phi$, and $C$ a constraint on the logvars of $\mathcal{B}$.
	A set of parfactors forms a \emph{model} $G := \{g_i\}_{i=1}^n$.
	With $Z$ as normalizing constant, $G$ represents the full joint distribution $P_G = \frac{1}{Z} \prod_{f \in gr(G)} f$, with $gr(G)$ referring to the groundings of $G$ w.r.t. given constraints.
    A grounding is the instantiation of each logvar in each parfactor with an allowed constant.
\end{definition}

Continuing Example~\ref{example:prvs}, let us define the parfactor model:

\begin{example}
Let $\forall m \in M: \phi(Sick(m),Epidemic)_\top$ be the parfactor with potential $\phi$ defining the probability to be sick for all persons from $\mathbf{D}$, given there is an epidemic (or not).
The grounded model then consists of the eight factors $\phi(Sick(a),Epidemic)$, \dots, $\phi(Sick(h),Epidemic)$ compared to one parfactor in the lifted model.
Table~\ref{table:potential} shows the potential function.
\end{example}

\begin{table}[tbp]
	\centering
    \caption{Potential function giving the probability of a person being sick given whether there is an epidemic.}
	\label{table:potential}
	\begin{tabular}{cc|c}
		$Epidemic$ & $Sick(M)$ \\ 
		\hline
		0 & 0 & 0.9 \\
		0 & 1 & 0.1 \\
		1 & 0 & 0.3 \\
		1 & 1 & 0.7 \\
	\end{tabular}
\end{table}

Next, we integrate parfactors into fMDPs to achieve a compact representation of state and action spaces.

\section{Relational Factored MDPs}
Now, we present the first-order representation for the state and the action space as well as the transition function, which allows to efficiently represent indistinguishable objects as well as concurrent action, which then in turn can be exploited by an appropriate algorithm.
To that end, we present Foreplan in the next section.

We define rfMDPs based on fMDPs, but include a parfactor model inside the state and action space and transition function.
That is, the set of state variables can now contain PRVs, which are then used in the transition model and the reward function.
Also, we have \emph{action PRVs}, whose value is chosen by the agent.
Basically, we keep the same semantics as in Definition~\ref{definition:mdp} and only change the representation to exploit indistinguishable variables.
In the following definition we use the term \emph{interpretation} of a set for a truth-value assignment to each element in the set.

\begin{definition}[Relational Factored MDPs] \label{definition:rfMDP}
A \emph{relational factored MDP} is a tuple $(\mathbf{D}, \mathbf{L}, \mathbf{B}, \mathbf{A}, G, \mathbf{R})$.
The set $\mathbf{D}$ is a set of constants and the set $\mathbf{L}$ is a set of logvars over $\mathbf{D}$.
The set $\mathbf{B}$ is a set of PRVs defined over $\mathbf{L}$.
The set of possible interpretations $\mathcal{I}_\mathbf{B}$ for the groundings of the set $\mathbf{B}$ defines the state space, i.e., all possible assignments.
The set $\mathbf{A}$ is a set of action PRVs.
A parfactor model $G$ over $\mathbf{A}$ and $\mathbf{B}$ represents the transition function $T: \mathcal{I}_\mathbf{B} \times \mathcal{I}_\mathbf{A} \times \mathcal{I}_\mathbf{B} \to \mathbb{R}_0^+$, with the set $\mathcal{I}_\mathbf{A}$ of possible interpretations of the groundings of $\mathbf{A}$, and specifies the transition probability given an action and a previous state.
The set $\mathbf{R}$ contains \emph{parameterized} local reward functions $R_i : \bigtimes_j \mathcal{R}(B_{i,j}) \to \mathbb{R}$, defined over all PRVs $B_{i,j}$, which contribute to the local reward $R_i$.
The reward function $R$ is decomposed as a sum operation over $\mathbf{R}$.
\end{definition}

Let us generalize Example~\ref{ex:sick} to an arbitrary number of persons behaving in the same way as the agent in Example~\ref{ex:sick}:
\begin{example}[Epidemic] \label{ex:epidemic}
There is a set $\mathbf{D}$ of persons living in a small town, represented by the logvar $M$ with $\mathcal{D}(M)=\mathbf{D}$.
Each person can be sick or healthy, leading to the PRV $Sick(M)$.
The government gets a reward of $1$ for each healthy person and a reward of $-1$ for each sick person. To combat an epidemic, the government can impose travel bans on persons, resulting in the action $Restrict(M)$ to impose a travel ban on a subset of persons.
Moreover, each person can travel or not, leading to the PRV $Travel(M)$. The government gets a reward of $2$ for each person travelling.
The PRV $Epidemic$ is influenced by the number of people travelling and influences the healthiness of each person.
Figure~\ref{fig:epidemic_grounded} shows the grounded transition model and figure~\ref{fig:epidemic_lifted} shows the lifted transition model for this example.
\end{example}

\begin{figure}[tbp]
    \centering
    \scalebox{0.7}{
    \begin{tikzpicture}[baseline=(current bounding box.center)]
        \node[node] (ta) {$Tr(a)$};
        \node[node] (sa) [right=of ta] {$Si(a)$};
        \node[node] (e) [right= of sa] {$Epi$};
        \node[node] (sb) [right=of e] {$Si(b)$};
        \node[node] (tb) [right=of sb] {$Tr(b)$};
    
        \node[block] (fta) [below=of ta] {$f_1$};
        \node[block] (fsa) [below=of sa] {$f_2$};
        \node[block] (fsb) [below=of sb] {$\tilde{f}_2$};
        \node[block] (ftb) [below=of tb] {$\tilde{f}_1$};
    
        \node[node] (ra) [left=of fta] {$Re(a)$};
        \node[node] (rb) [right=of ftb] {$Re(b)$};
    
        \node[node] (sap) [below=of fsa] {$Si'(a)$};
        \node[node] (tap) [below=of fta] {$Tr'(a)$};
        \node[node] (ep) [right=of sap] {$Epi'$};
        \node[node] (sbp) [below=of fsb] {$Si'(b)$};
        \node[node] (tbp) [below=of ftb] {$Tr'(b)$};
    
        \node[block] (fe) at ($(e)!0.5!(ep)$) {$f_3$};
    
        \draw[-] (ta) -- (fta);
        \draw[-] (ra) -- (fta);
        \draw[-] (tap) -- (fta);
        \draw[-] (tb) -- (ftb);
        \draw[-] (rb) -- (ftb);
        \draw[-] (tbp) -- (ftb);
        \draw[-] (sa) -- (fsa);
        \draw[-] (e) -- (fsa);
        \draw[-] (sap) -- (fsa);
        \draw[-] (sb) -- (fsb);
        \draw[-] (e) -- (fsb);
        \draw[-] (sbp) -- (fsb);
        \draw[-] (ta) -- (fe);
        \draw[-] (tb) -- (fe);
        \draw[-] (ep) -- (fe);
    \end{tikzpicture}
    }
    \caption{Grounded transition model for Example~\ref{ex:epidemic}. We assume that the population of the town consists of two citizens $a$ and $b$. For abbreviation, we use only the first letter(s) for each symbol. The factors $f_1$ and $\tilde{f}_1$ as well as $f_2$ and $\tilde{f}_2$ could be combined into a parfactor. Figure~\ref{fig:epidemic_lifted} shows the lifted version.}
    \label{fig:epidemic_grounded}
    \Description{Grounded graphical model for the epidemic example: The factor f_1 links Travel(a) and Reestrict(a) to Travel'(a), analogously \tilde{f}_1 for b, f_2 links Sick(a) and Epidemic to Sick'(a), analogously \tilde{f}_2 for b, and f_3 links Travel(a) and Travel(b) to Epi'.}
\end{figure}

\begin{figure}[tb]
    \centering
    \scalebox{0.7}{
        \begin{tikzpicture}[baseline=(current bounding box.center),pf/.style={draw, rectangle, fill=gray}]
            \node[node] (ta) {$Tr(M)$};
            \node[node] (sa) [right=of ta] {$Si(M)$};
            \node[node] (e) [right= of sa] {$Epi$};

            \pfs{below}{ta}{40}{$f_1$}{f1i}{fta}{f1o};
            \pfs{below}{sa}{130}{$f_2$}{f2i}{fsa}{f2o};
        
            \node[node] (ra) [left=of fta] {$Re(M)$};
        
            \node[node] (sap) [below=of fsa] {$Si'(M)$};
            \node[node] (tap) [below=of fta] {$Tr'(M)$};
            \node[node] (ep) [right=of sap] {$Epi'$};
        
            \pfsat{$(e)!0.5!(ep)$}{40}{$f_3$}{f3i}{fe}{f3o};

            \draw[-] (ta) -- (fta);
            \draw[-] (ra) -- (fta);
            \draw[-] (tap) -- (fta);
            \draw[-] (sa) -- (fsa);
            \draw[-] (e) -- (fsa);
            \draw[-] (sap) -- (fsa);
            \draw[-] (ta) -- (f3i);
            \draw[-] (ep) -- (fe);
        \end{tikzpicture}
    }
    \caption{Lifted representation of the transition model for Example~\ref{ex:epidemic}. We abbreviate by using only the first letter(s) for each symbol. Figure~\ref{fig:epidemic_grounded} shows the lifted representation.}
    \label{fig:epidemic_lifted}
    \Description{Parameterized graphical model for the epidemic example: The factor f_1 links Travel(M) and Restrict(M) to Travel'(M), f_2 links Sick(M) and Epidemic to Sick'(M) and f_3 links Travel(M) to Epidemic'.}
\end{figure}

Since an action can be applied to each person concurrently, the amount of possible actions is exponential due to the power set, i.e., all possible combinations for all persons.
Foreplan avaoids the exponential blow-up by exploiting the indistinguishability of action PRVs.

With Definition~\ref{definition:rfMDP}, we require the rewards to be represented as a sum of local reward functions, and not as some arbitrary function.
Actually, this is too restrictive:
It is sufficient that an operation running over same behaving variables has a hyperoperation repeating the operation multiple times.
For a sum, the hyperoperation is multiplication.
Also, multiple terms of such operations can be linked together in any way.
In Appendix~\ref{appendix:hyperoperations}, we give examples for hyperoperations.
For simplicity and the proofs, we use the definition as it is throughout the paper.

We exploit the symmetries in Definition~\ref{definition:rfMDP} with Foreplan in the next section.
In the remainder of this section, we explain what we mean by \emph{action PRVs} and by \emph{parameterized} local reward functions.

\subsection{Parameterizing Actions}
In our epidemic example, the mayor, representing the government, can impose travel bans on all parts of the town's population.
We extend the action definition in this subsection to account for groups of objects.
Having groups, we circumvent enumerating of all possible subsets and model the example action of imposing travel bans on a subset of the population efficiently.

\begin{definition}[Action PRV] \label{def:parameterized_action}
An action PRV $A$ is a Boolean-valued PRV.
A concrete action is a set of events, i.e., each grounding of $A$ receives an assignment $a \in \mathcal{R}(A)$.
\end{definition}

Action PRVs allow for a more general action setting.
When writing action, we refer to Definition~\ref{def:parameterized_action}. %
In our example, the mayor can restrict multiple persons from travelling at once:

\begin{example}[Action PRV] \label{ex:parameterized_action}
The action PRV $Restrict(M)$ models the possible travel bans on the population of the town.
For a concrete action, the mayor has to specify which persons are restricted from travelling.
When restricting the persons $a$, $b$ and $f$, the mayor specifies $Restrict(M)=true$ with constraint $(M, \{a,b,e\})$ and $Restrict(M) = false$ with constraint $(M, \{c,d,f,g,h\})$.
\end{example}

However, when the persons $a$ to $h$ are indistinguishable, it is irrelevant which selection is used for the constraints, only the amounts are relevant.
We describe the impact of parameterization on the rewards next.

\subsection{Parameterized Local Reward Functions}
The reward function in fMDPs maps from the (joint) state to the reward of the state.
For evaluating the reward function, we thus have to construct the joint state and cannot exploit the factorization,
breaking our aim of compact representation.
To further use our compact representation, we introduce a decomposable reward function:
For simplicity, we assume that the reward function is factored as $R = \sum_i R_i$, with local reward functions $R_i$ with scope restricted to a subset of the state variables.
Other operations are possible too, as long as they have a hyperoperation, as already mentioned.
As we have indistinguishable state variables, we reduce redundant computations by using representatives in the reward functions analogously to a parfactor, but using a sum instead of a product:

\begin{definition} %
A local reward function $R_i:\bigtimes_j \mathcal{R}(B_{i,j}) \to \mathbb{R}$ is defined over all PRVs $B_{i,j}$, which contribute to the local reward $R_i$.
The semantics of a single parameterized local reward function $R_i$ is defined as the sum $\sum_z R_i(z)$ over the interpretations $z \in \bigtimes_j \mathcal{R}(B_{i,j})$ in the current state of all groundings of $R_i$.
\end{definition}

In other words, a parameterized local reward function serves as a placeholder for the set of local reward functions, obtained by replacing all logvars by their possible instantiations.
We illustrate parameterized reward functions in our epidemic example:
\begin{example} \label{ex:lifted_reward} %
The parameterized local reward functions for Example~\ref{ex:epidemic} are $R_1(Sick(M))$, evaluating to $-1$ ($1$) for each person (not) being sick, and $R_2(Travel(M))$, evaluating to $2$ for each person travelling.
If five persons are sick, three are not sick and four people are travelling, the total reward is $-5 + 3 + 8 = 6$.
\end{example}

With rfMDPs, we can efficiently represent the action and state space as well as the transition function for numerous indistinguishable objects.
Further, we have a compact representation for concurrent actions.
In the next section, we propose Foreplan, our main contribution, which uses our representation to solve planning in rfMDPs in time polynomial in the number of objects.

\section{Foreplan: Efficient Forward Planning}
In this section, we propose Foreplan, our exact \underline{fo}rward \underline{re}lational \underline{plan}ner for rfMDPs.
The input to Foreplan is an rfMDP.
The output is a value function, which induces a policy.
Foreplan computes the value function using a compact state representation for the rfMDP to exploit the indistinguishability of the objects and then running a linear program based on the state representation to calculate the value function.

The first-order representation in rfMDP contains dependencies between PRVs.
For using the representation in computation, we first need to recognize the dependencies and find a suitable representation of them to compute the value function efficiently.
We first describe how Foreplan recognizes the dependencies and identifies a suitable representation.
Afterwards, we outline how Foreplan uses the identified representation to compute a value function.

\subsection{Obtaining a Compact State Representation} \label{sec:state_representation}
Foreplan needs to encode the current state compactly to efficiently reason about indistinguishable variables.
Thus, in this subsection, we describe how Foreplan treats indistinguishable objects for computations using rfMDPs.
Foreplan does not need to keep track of objects that could be differentiated by their history.
Rather, with each action and new time step, the history is swept away and the objects remain indistinguishable because of the first-order Markov assumption.
To compactly represent the state, the basic idea is to \emph{count} the objects for which the PRVs are true or false, using the idea of Counting Random Variables (CRVs)~\cite{crv}.
When evaluating a parfactor, the joint assignment to the PRVs is needed, rather than individual counts for each PRV.
It is sufficient to count the number of occurrences of each possible truth-value assignment to the groundings of the \emph{input} PRVs of a parfactor in a histogram.
The \emph{input} PRVs of a parfactor are the ones representing the current state.
Counting only the input PRVs is sufficient, because all possible next states are iterated separately later. %
But some parfactors may be defined over PRVs defined over different logvars, like $f_3$ in Figure~\ref{fig:epidemic_lifted} is defined over $Travel(M)$ and $Epidemic$.
Given the counts for $Travel(M)$ and $Epidemic$ separately is enough, since these PRVs do not depend on each other.
But how does Foreplan recognize which PRVs depend on which PRVs?
For answering this question, we propose the \emph{relational cost graph} to identify the dependencies between PRVs and identify the groups that need to be counted jointly.
Focusing only on the counts for the groups enables Foreplan to use a much simpler state space representation, namely the set of possible histograms, compared to the grounded state space representation, which is the Cartesian product over the domains of all state variables.
We now describe how to count the assignments in more detail.

Counting the assignments of each PRV separately is insufficient, as PRVs can be defined over the same logvars and thus interfere with each other.
However, the parfactors can be evaluated separately since, in Equation~\ref{eq:lp}, we have the full current and next state available.
Thus, it is sufficient to count PRVs together if they share a logvar and occur together in a parfactor.
To obtain the representation and later on quantify its complexity, we define the relational cost graph:
\begin{definition}[Relational Cost Graph] \label{def:relational_cost_graph}
The \emph{relational cost graph} of an rfMDP has a vertex for each PRV in the current state.
Two vertices are connected by an edge if and only if the PRVs associated with these two vertices share a logvar and occur together in a parfactor as input PRVs or a parameterized local reward function.
We denote the number of (maximal) cliques by $c$ and the size of the largest clique by $w$.
\end{definition}

We give a more formal definition of the relational cost graph in Appendix~\ref{appendix:rcg_formal_definition}.
Let us take a look at the relational cost graph of Example~\ref{ex:epidemic}.
For a more complex example, we refer to Appendix~\ref{appendix:rcg-examples}.
\begin{example} \label{ex:epidemic_relational_cost_graph}
The relational cost graph for Example~\ref{ex:epidemic} consists of three isolated vertices corresponding to $Sick(M)$, $Travel(M)$ and $Epidemic$. The first two do not occur together in a parfactor or local reward function and both do not share a logvar with the last one.
Figure~\ref{fig:epidemic_rcg} shows the visual representation of the relational cost graph.
\end{example}

\begin{figure}[tbp]
    \centering
    \scalebox{0.7}{
    \begin{tikzpicture}[baseline=(current bounding box.center),
                        nodesize/.style={minimum size=1.6cm}]
        \node[node, nodesize] (s) {$Si(M)$};
        \node[node, nodesize] (t) [right= of s] {$Tr(M)$};
        \node[node, nodesize] (e) [right= of t] {$Epi$};
    
    \end{tikzpicture}
    }
    \caption{Visual representation of the relational cost graph in Example~\ref{ex:epidemic}. We abbreviate using the first letter(s).}
    \label{fig:epidemic_rcg}
    \Description{The relational cost graph for the epidemic example consists of three isolated vertices corresponding to Sick(M), Travel(M) and Epidemic.}
\end{figure}

The key insight now is that (maximal) cliques in the relational cost graph correspond to sets of PRVs that Foreplan needs to count together as they interfere with each other.
For one logvar per PRV, basic CRVs~\cite{crv} already return the result.
To lift the limitation of one logvar per PRV, we extend the definition by \customcite{crv}:

\begin{definition}[Extended Counting Random Variable] \label{def:extended_crv}
    A counting formula $\gamma = \#_C[B_1,\dots,B_k]$ is defined over PRVs $B_i$ with a constraint $C = (\mathcal{L}, C_{\mathcal{L}})$ over the logvars $\mathcal{L}$ of the PRVs $B_i$.
    The counting formula represents a \emph{counting random variable} (CRV) whose range is the set of possible histograms that distribute $n = \abs{C_{\mathcal{L}}}$ elements into $\prod_{i=1}^k \abs{\mathcal{R}(B_i)}$ buckets.
    The state of $\gamma$ is the histogram function $h = \{(r_i,n_i)\}_{i=1}^r$ stating for each $r_i \in \times_{i=1}^k \mathcal{R}(B_i)$ the number $n_i$ of tuples $(B_i)_i$ whose state is $r_i$.
\end{definition}

If no restrictions apply, we omit $C$, and $n = \prod_{L \in \mathbf{L}} \abs{\mathcal{D}(L)}$, where $\mathbf{L}$ is the set of logvars of the PRVs $B_i$.
We do not define the operations to manipulate CRVs as we do not need them in this paper.
The CRV corresponding to a clique gives us the number of occurrences for each possible instantiation of the PRVs in that clique.
For illustrative purposes, we give a small example for a CRV using an additional PRV $Friends(M,Y)$, with $M$ and $Y$ having the same domain:
\begin{example} \label{ex:crv}
To have a PRV with two logvars, assume we have three PRVs, $Travel(M)$, $Friends(M,Y)$ and $Sick(Y)$ with $\abs{\mathcal{D(M)}}=\abs{\mathcal{D}(Y)}=8$.
A possible state for the CRV $\#[Travel(M)$, $Friends(M$, $Y)$, $Sick(Y)]$ is
$\{(ttt,6)$, $(ttf,30)$, $(tft,4)$, $(tff,0)$, $(ftt,6)$, $(ftf,0)$, $(fft,0)$, $(fff,18)\}$ and we have $n = 64$.
The first histogram entry shows that, in this state, there are six tuples for which $Travel(m)=t$, $Friends(m,y)=t$, and $Sick(y)=t$ holds.
\end{example}

Example~\ref{ex:crv} illustrates that we need only eight buckets in a histogram regardless of the domain size of $M$ and $Y$.
The number of buckets depends on the range values of the PRVs and not on the domain sizes.
We formalize the state representation, which is a set of CRVs per (maximal) clique in the relational cost graph: %
\begin{definition}[State Representation] \label{def:state-representation}
For counting the corresponding PRVs in a clique, we create one CRV $S_i$ for each (maximal) clique in the relational cost graph.
For a single propositional RV in the relational cost graph, we use the RV for $S_i$ instead of a CRV.
A \emph{state} assigns a value to each $S_i$.
The resulting \emph{state space} is the set of possible states.
We denote the \emph{state space} by the tuple $(S_i)_i$, which contains one CRV (or RV) per clique.
\end{definition}

We give the representation of the state space for Example~\ref{ex:epidemic}:
\begin{example} \label{ex:state_representation}
As the vertices are not connected in the relational cost graph, so the state representation is $(\#[Sick(M)]$, $\#[Travel(M)]$, $Epidemic)$.
\end{example}

The relational cost graph in Example~\ref{ex:epidemic_relational_cost_graph} tells us that we can count the number of healthy persons separately from the number of travelling persons.
In particular, Foreplan does not store which sick persons are travelling, since this information is irrelevant.
We prove that our state representation is correct, i.e., retains the same semantics, namely exactly representing $\mathbf{S}$ of a ground MDP:
\begin{theorem} \label{theorem:state-representation-correctness}
    The representation in Definition~\ref{def:state-representation} is correct.
\end{theorem}
\begin{proof}[Proof Sketch]
    Given groundings for the state PRVs, we derive the histograms for the CRVs by counting the assignments for each parfactor.
    Given a representation as in Definition~\ref{def:state-representation}, we reconstruct, for each parfactor, the groundings of the PRVs by using the counts from the CRVs to instantiate the respective parfactor.
    We provide a full proof in Appendix~\ref{appendix:state-representation}.
\end{proof}

To advance through an action to the next state, the action now has to use the same state representation, i.e., the action is specified on the counts for all PRVs of the current state for all parfactors the action is mentioned in:

\begin{example} \label{ex:action_with_state}
The action $Restrict(M)$ uses the PRV $Travel(M)$. %
Thus, the mayor needs to specify how many persons of those (not) travelling are (not) allowed to travel. %
Suppose that in the current state five out of eight persons are travelling and two persons are sick.
The action $Restrict(M)$ is defined over $\#[Travel(M),Restrict(M)]$.
A concrete action is, e.g., $a=\{(tt,3)$, $(ft, 2)\}$, which means that three persons currently travelling are restricted from travelling and two persons not travelling.
The action $a$ does not need to specify the counts of people no travel ban is imposed on ($tf$, $ff$), as these are determined by $a$ and the current state.
\end{example}

The mayor no longer needs to specify individual persons (c.f. Example~\ref{ex:parameterized_action}), but rather the number of persons (not) travelling, which are restricted from travelling.
It is irrelevant on which exact persons the action is performed.
With this action representation, we reduce the action space from exponential to polynomial, which we prove in Theorem~\ref{theorem:state-representation-size-bounded-by-rcg} in the next section.

In the next subsection, we show how Foreplan uses the action space to compute the value function by solving a linear program.

\subsection{Computing the Value Function}
Let us have a look on how Foreplan computes the value function based on the introduced state representation.
Foreplan uses the linear programming formulation given in Equation~\ref{eq:lp} to compute the value function.
For the linear program, Foreplan uses the introduced state and action representations to iterate over all states and actions.

For instantiating the linear program in Equation~\ref{eq:lp}, we next describe how Foreplan calculates the transition probability.
Since we have full evidence provided, Foreplan evaluates each state CRV $S_i'$ separately.
For a fixed state CRV $S_i'$, the value $s_i'$ is fixed since the whole state space is iterated.
For evaluating $P(s_i' \mid s,a)$, Foreplan iterates over all possible assignment transitions, e.g., from each bucket in the histogram $s$ to each bucket in the histogram $s'$.
In our epidemic example, this is, e.g., the number of people getting sick and healthy.
For each possible assignment transition, Foreplan calculates the transition probability:
The probability for a representative transition is given by the involved parfactors, i.e., the parfactors defined over the PRVs mentioned in $S_i'$.
But this representative probability has to be weighted by how many times this assignment transition is applicable using the multinomial coefficient.
Take for example the state in which five persons are sick.
When three sick persons get healthy, this assignment transition has weight $5 \choose 3$ as any three of the five persons can get healthy.
The final probability of $P(s_i' \mid s,a) = \sum_t w_t p_t$ for a fixed CRV is the sum
over all possible assignment transitions $t$, each with representative probability $p_t$ and weight $w_t$.
Since the state space is factored, the full transition probability $P(s' \mid s, a)$ is factored as $P(s' \mid s, a) = \prod_i P(s_i' \mid s_i, a)$.
We provide a more detailed description and an example in Appendix~\ref{appendix:transition_probabilities}.

Let us have a closer look at the actual creation of Foreplan's linear program:
For each state and action combination, a constraint is added.
Thus, Foreplan can take both, state and action, into account for generating only the necessary constraints.
When an initial state is given, the reachable state space can be pruned, and Foreplan limits all computations to the pruned search space, that is why we are calling Foreplan a \emph{forward} planner.
Moreover, additional checks like mutual exclusion, capacities, or other constraints can be added, either on PRV level or on search space level.
We make this more concrete with an example:
Assume we have four persons travelling.
Then, since the action uses the lifted state representation, the mayor can only restrict up to four persons, which currently travel, from travelling.
The generation of constraints for the assumed state only iterates over the possible actions.
Future work includes an analysis of what types of additional applicability checks are efficient implementable.

With Foreplan, we are able to cope with numerous indistinct objects and actions on collections of those objects. We do so by successfully applying lifting in the field of MDPs.
While traditional approaches can represent actions on sets of objects, they fail to do so efficiently.
Therein, the actions for each subset would be represented on their own, resulting in exponentially many actions.
In the next section, we show the complexity of Foreplan.

\section{Complexity Analysis of Foreplan} \label{sec:foreplan_theorems}
Having outlined Foreplan, we analyze the complexity of Foreplan in this section.
We start by quantifying the state representation and using the complexity of the state representation to derive the runtime complexity of Foreplan.

We derive the following theorem about the size of the state representation from Definition~\ref{def:state-representation}.
\begin{theorem} \label{theorem:state-representation-size}
The state representation is in $\mathcal{O}(c \times 2^w)$.
\end{theorem}
\begin{proof}
For each clique in the relational cost graph, the size of the histogram function is exponential in the number of vertices in the clique, as we enumerate all possible assignments.
Thus, the size of the state representation is bounded by $c \times 2^w$.
\end{proof}

Note that $c$ and $w$ are independent of the domain sizes and determined only by the structure of the relational cost graph, and thus in general small.
Also $w$ is bounded by the number of PRVs and $c$ by the number of parfactors in the parfactor model.
Theorem~\ref{theorem:state-representation-size} overapproximates the size of the state representation, as not all cliques have the same size and not both, $c$ and $w$ are large at the same time.
Building on the size of the state representation, we give the complexity of the state and action space:
\begin{theorem} \label{theorem:state-representation-size-bounded-by-rcg}
    The state and action spaces are both polynomial in the number of objects and exponential in $c$ and $w$.
\end{theorem}
\begin{proof}
    We need to iterate over all possible instantiations of the state representation.
    For each clique (resp. CRV), the number of possible instantiations is polynomial in the number of objects and exponential in $w$.
    The joint state requires one instantiation per clique, resulting in the number $c$ of CRVs as exponent.
    The size of the action space is bounded by the size of the state space, as the action has to specify a (subset of a) state.
\end{proof}

Since Foreplan uses a linear program to compute the value function, we analyze the complexity of solving the linear program Foreplan builds.
Linear programs can be solved in polynomial time w.r.t the variables and constraints~\cite{63499}.
Let us therefore take a closer look at the number of constraints and variables Foreplan generates:

\begin{theorem} \label{theorem:lp_number_constraints_variables}
    The number of linear programming constraints and variables Foreplan creates are polynomial in the state space.
\end{theorem}
\begin{proof}
By Eq.~\ref{eq:lp}, Foreplan generates one variable per possible state and one constraint for each state and action combination.
\end{proof}

Plugging Theorem~\ref{theorem:state-representation-size-bounded-by-rcg} into Theorem~\ref{theorem:lp_number_constraints_variables} leads to:
\begin{theorem} \label{theorem:foreplan-complexity}
The runtime of Foreplan is polynomial in the number of objects and exponential in $c$ and $w$.
\end{theorem}

Therefore, with Foreplan, we have introduced an exact planner with a runtime polynomial instead of exponential w.r.t domain sizes for concurrent actions.
Thus, we have already achieved an exponential speed up.

Our main contribution, Foreplan, significantly advances the state of the art.
However, as it is an exact algorithm, the value function is still not factorized.
In case the runtime is of utter importance, we can also approximate the value function, following the idea of \customcite{approx_factored_mdps}, which has also been picked up by, e.g., \customcite{10.5555/3020336.3020398}.
By approximating the value function, we can prevent iterating over the joint state space leading to a blazingly fast, but approximate, version of Foreplan.

\section{Foreplan: Even Faster by Approximation}
While Foreplan runs in time polynomial in the number of objects, the runtime still exponentially depends on $c$.
In this section, we present an approximation technique inspired by the Approximate Linear Programming (ALP) approach~\cite{approx_factored_mdps} to prevent iterating the whole state space.
Our approach follows the same idea as ALP and ALP for first-order MDPs~\cite{10.5555/3020336.3020398}.
We first describe the approximation idea and then how Foreplan uses the approximation for our new representation and for concurrent actions.
Last, we give bounds on the runtime and on the approximation quality.
We call the approximate version \emph{Approximate Foreplan}.
Since Approximate Foreplan is an extension of our main contribution Foreplan, we refer to the Appendix for technical details.

Foreplan needs to iterate over the whole state space, because the value function maps from a state to the value of that state.
We approximate the value function by a set $h_i$ of basis functions, whose scope is a subset of $\mathbf{S}$, $V \approx \sum_i w_i h_i$, where the goal is to find the most suitable weights $w_i$~\cite{approx_factored_mdps}.
Approximate Foreplan also needs the value of all possible next states in terms of the same approximation.
Thus, Approximate Foreplan uses \emph{backprojections} $g_i^a$ of the basis functions $h_i$~\cite{approx_factored_mdps}, stating the influence of $x$ on the next state:
\begin{equation}
g_i^a(x) = \sum_{x'} P(x' \mid x,a) \cdot h_i(x')
\end{equation}
To compute the basis functions and backprojections in a lifted way, we parameterize the basis functions analogously to the reward functions and calculate them in the same way.
The basis functions should capture the important dynamics in the model~\cite{koller1999computing}, most importantly the rewards~\cite{10.5555/3020336.3020398}.
As proposed by \customcite{koller1999computing} and picked up in \customcite{10.5555/3020336.3020398}, we also use basis functions for capturing the rewards alongside a constant basis function:
\begin{example}[Basis Function] \label{ex:lifted_basis_function}
We have three basis functions: $h_0$ $\coloneq$ $1$, $h_1(Sick(M))$ $\coloneq$ $R_1(Sick(M))$ as well as $h_2(Travel(M))$ $\coloneq$ $R_2(Travel(M))$.
\end{example}
The backprojections are computed in a lifted way:
\begin{definition}[Lifted Backprojection] \label{def:lifted-backprojection}
Given a basis function $h_i$ and Boolean assignments $\tilde{x}$ and $\tilde{a}$ to the state and action, respectively, the backprojection is defined as $g_i^{\tilde{a}}(\tilde{x}) = \sum_{\tilde{x}'} P(\tilde{x}' \mid \tilde{x},\tilde{a}) \cdot h_i(\tilde{x}')$.
The \emph{lifted backprojection} $G_i^a(x)$ for a state $x$ and action $a$ then sums $g_i^{\tilde{a}}(\tilde{x})$ for each possible propositional assignment $\tilde{x}$ and $\tilde{a}$ and weights the term with the counts given by the state $x$.
\end{definition}

Let us apply the backprojection in our running example:

\begin{example}[Lifted Backprojection]
Suppose we have three sick persons and two healthy ones and are interested in the backprojection of $h_1$.
Then, we have $G_1([(t,3),(f,2)],epi)=3 \cdot g_1(true,epi) + 2 \cdot g_1(false,epi)$.
We show the full calculation of all backprojections for our running example in Appendix~\ref{apx:walkthrough-backprojections}.
\end{example}

Approximate Foreplan precomputes all backprojections and then builds the following linear program~\cite{approx_factored_mdps}, whose result is the approximated value function and thus constitutes the result: %

\begin{equation} \label{eq:factored_lp}
\begin{array}{ll@{}}
\text{Variables:} & w_1,\dots,w_n \ ;\\
\text{Minimize:} & \displaystyle\sum_{i=1}^n \alpha_i w_i \ ;\\
\text{Subject to:} & \forall a \in A: \\
& 0 \geq \displaystyle\max_x  \left\{R(x) + \displaystyle\sum_{i=1}^n w_i \left(\gamma G_i^a(x) - h_i(x)\right)\right\}. \\
\end{array}
\end{equation}
The $\alpha_i$'s are effectively coefficients for a linear combination over the $w_i$, stating how important the minimization of each $w_i$ is~\cite{approx_factored_mdps,doi:10.1287/opre.51.6.850.24925}.
The maximum operator is no operator in linear programs and is removed in an operation similar to variable elimination (VE)~\cite{ve}.
In Appendix~\ref{appendix:ve}, we provide an example for the removal procedure.
Moreover, in Appendix~\ref{apx:walkthrough}, we show a walkthrough of Approximate Foreplan on the epidemic example.

For the runtime analysis, we introduce the \emph{cost network} briefly:
The cost network for a constraint has a vertex for each appearing variable and there is an edge in the cost network between two vertices if the corresponding variables appear together in the same reward or basis function, or backprojection.
For the complexity analysis, we use the structural graph parameter \emph{induced width}, which provides an upper bound for the largest intermediate result~\cite{dechter1999bucket}:

\begin{theorem} \label{theorem:approximate:relational_cost_network_bounded_clique}
Approximate Foreplan runs in time polynomial in the number of objects, polynomial in $c$ and exponential in the induced width of each cost network, when $w$ is bounded.
\end{theorem}
\begin{proof}[Proof Sketch]
The full proof is in Appendix~\ref{appendix:approx_foreplan_runtime}.
Approximate Foreplan has to precompute the backprojections and solve the linear program in Equation~\ref{eq:factored_lp}.
The first involves constantly many iterations of the state and action spaces.
The second is linear in the action space and exponential in the induced width of each cost network~\cite{approx_factored_mdps,dechter1999bucket}.
Because Approximate Foreplan does not iterate over the whole state space, but treats each clique independently in the maximum operator, the effective state space is no longer exponential in $c$, but polynomial, and the action space is bound by the state space.
\end{proof}

Most notably, $w$ and induced width in Theorem~\ref{theorem:approximate:relational_cost_network_bounded_clique} are mostly small and fixed, leading to a polynomial runtime, as the growth in the number of objects is of more interest.
Combining the relational cost graph and the cost networks in a single \emph{total relational cost graph}, we show in Appendix~\ref{appendix:foreplan-theorems} that the runtime of Approximate Foreplan is polynomial in the number of objects when the treewidth of the total relational cost graph is bounded.

Moreover, we can show that Approximate Foreplan and ALP compute the same solutions:

\begin{theorem} \label{theorem:equivalence_alp_aforeplan}
Given  an rfMDP $R$, Approximate Foreplan and ALP are equivalent on $R$ and the grounded version of $R$, respectively.
\end{theorem}
\begin{proof}[Proof Sketch]
The full proof is in Appendix~\ref{appendix:foreplan-approximation}.
With appropriate $\alpha_i$, the objective function carries out lifted computation of the grounded basis functions. %
Each constraint is correct, because the lifted backprojections and lifted basis functions compute the same value as for the grounded model.
The action representation covers the whole action space.
\end{proof}

With ALP and Approximate Foreplan being equivalent, we transfer the approximation guarantee for ALP to Approximate Foreplan:

\begin{corollary}[Approximation Guarantee~\cite{approx_factored_mdps,doi:10.1287/opre.51.6.850.24925}]
Approximate Foreplan provides the best approximation of the optimal value function in a weighted $\mathcal{L}_1$ sense, where the weights in the $\mathcal{L}_1$ norm are the state relevance weights $\alpha$.
\end{corollary}
\begin{proof}
Since the claim holds for ALP~\cite{approx_factored_mdps,doi:10.1287/opre.51.6.850.24925}, the proof follows by the equivalence of ALP and Approximate Foreplan.
\end{proof}

With Approximate Foreplan, we reduce the runtime further from exponential in $c$ to polynomial in $c$.
In the next section, we evaluate the runtimes of Foreplan and Approximate Foreplan empirically.

\section{Empirical Evaluation}
(Approximate) Foreplan runs in time polynomial in the number of objects, but other terms are unavoidably  exponential.
In contrast to current approaches, the exponential terms of both Foreplan variants depend only on the structure of the rfMDP and not on the number of objects.
To underline our theoretical results, we evaluate (Approximate) Foreplan against ALP and an implementation of symbolic value iteration (VI) using extended algebraic decision diagrams (XADDs)~\cite{pmlr-v162-jeong22a,taitler2022pyrddlgym} for the epidemic example introduced in Example~\ref{ex:epidemic} as well as for the BoxWorld~\cite{boutilier2001symbolic} and a fully-connected SysAdmin~\cite{approx_factored_mdps} example, with the latter two extended to concurrent actions.
We also assess the quality of the policy given by Approximate Foreplan.
We use Python 3.12 and HiGHS for solving the linear programs~\cite{highs}.
We run all implementations on an AMD EPYC 7452 32-Core Processor with 504 GB of RAM.

Figure~\ref{fig:epidemic_total_runtime} shows the runtime for the epidemic example for (Approximate) Foreplan, ALP and XADD Symbolic VI for up to 191 persons with a time limit of two hours.
XADD Symbolic VI exceeds the time limit after ten persons, ALP after 15 persons.
Foreplan times out after 20 persons and Approximate Foreplan after 191 persons.
For ten persons, Foreplan is $91$ times faster than XADD Symbolic VI.
Approximate Foreplan is even $14.000$ times faster, which are four orders of magnitude.
For 15 persons, Foreplan is more than four times faster than ALP and Approximate Foreplan is more than $4.217$ times faster than ALP.
For 20 persons, Approximate Foreplan is more than $4.077$ times faster than Foreplan.
Thus, when using a symbolic solver, we can only solve the epidemic example for up to ten persons.
With a factored and approximate approach, we can go up to 15 persons.
In contrast, when using Foreplan, we can solve the epidemic example even for 20 persons and with Approximate Foreplan we can go further to 191 persons, which is ten times more than what ALP can solve.

We also run the four algorithms on a fully connected SysAdmin example with a timeout of two hours.
XADD Symbolic VI times out after eleven computers and ALP after 12.
With Foreplan, we can go up to 64 computers and with Approximate Foreplan even up to 94 computers.
At eleven computers, Foreplan is more than $44.413$ times faster than XADD Symbolic VI and Approximate Foreplan is even more than $69.944$ times faster.
At 12 computers, Foreplan is more than $8.359$ times faster than ALP and Approximate Foreplan even more than $16.173$ times.
For 64 computers, Approximate Foreplan is more than $696$ times faster than Foreplan.
In Appendix~\ref{appendix:evaluation}, we present the runtime figure.

For the BoxWorld example, we use three cities and set a time limit of 15 hours.
XADD Symbolic VI times out after two boxes and trains.
ALP and Foreplan time out after nine boxes and trains.
Approximate Foreplan manages to go up to 30 boxes and trains.
At two boxes and trains, Foreplan is more than $83.227$ times faster than XADD Symbolic VI and Approximate Foreplan even more than $137.502$ times.
At nine boxes and trains, Foreplan is more than three times faster than ALP and Approximate Foreplan is even more than $33.136$ times faster than ALP.
In Appendix~\ref{appendix:evaluation}, we present the runtime figure.

For assessing the quality of the policy given by Approximate Foreplan, we calculate the ratio of wrong actions over the total number of actions in every ground state.
For ten persons, ALP and Approximate Foreplan both deviate in $1.2$ \% of all actions.
For two to ten persons, the deviation was $2.98$ \% at most, with reducing errors when increasing the number of individuals.
In all cases, ALP and Approximate Foreplan have an identical error.
For SysAdmin, ALP and Approximate Foreplan both return the optimal policy in the tests going up to nine computers.
For BoxWorld, we refer to Appendix~\ref{appendix:evaluation}.

Overall, (Approximate) Foreplan achieves a speedup of several orders of magnitude and computes policies for significantly more objects within the same time and memory limits.
Moreover, the policy given by Approximate Foreplan is correct for SysAdmin and for more than $97$ \% of all actions in the epidemic example.

\begin{figure}[tb]
	\centering
	\includegraphics[width=\linewidth]{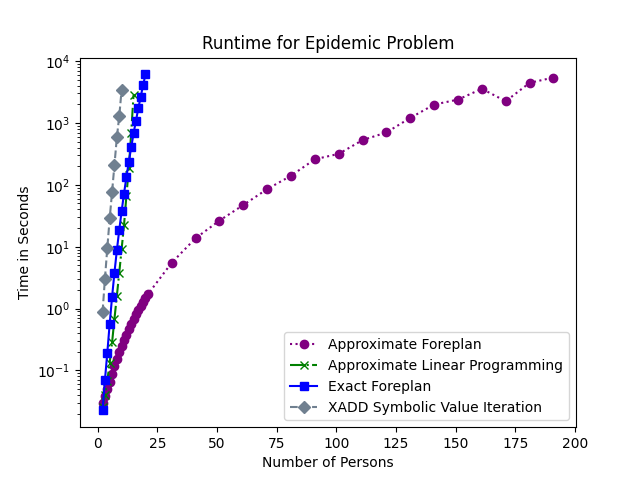}
	\caption{Runtime of (Approximate) Foreplan, ALP and XADD Symbolic VI on the epidemic example for up to 191 persons with a time limit of two hours.}
	\label{fig:epidemic_total_runtime}
    \Description{Graphical runtime comparison of Foreplan, Approximate Foreplan, Approximate Linear Programming and XADD Symbolic VI.}
\end{figure}

\section{Conclusion}
Propositional planning approaches struggle with having numerous indistinct objects and different types of actions applied concurrently. %
While first-order MDPs can cope with numerous objects, current approaches for solving planning in them neglect concurrent actions and still need to represent actions for each subset individually, resulting in exponentially many actions.
In this paper, we present Foreplan, a relational forward planner solving the exponential explosion by lifting the objects: %
Foreplan groups indistinguishable objects by a representative and carries out calculations on a representative-level. Afterwards, the result is projected to the whole group.
Using histograms and focusing only on the number of objects an action is applied to, we effectively reduce the action space from exponential to polynomial in the number of objects.

In future work, we aim to develop a hybrid approach combining Foreplan and Golog~\cite{levesque1997golog}. With the forward search in Foreplan, we can identify states reachable from the initial state while the backwards search in Golog computes the exact optimal policy.
Furthermore, the techniques from Foreplan can be transferred to first-order partially observable MDPs~\cite{williams2005factored}.

\begin{acks}
The research for this paper was funded by the Deutsche Forschungsgemeinschaft (DFG, German Research Foundation) under Germany's Excellence Strategy – EXC 2176 'Understanding Written Artefacts: Material, Interaction and Transmission in Manuscript Cultures', project no. 390893796. The research was conducted within the scope of the Centre for the Study of Manuscript Cultures (CSMC) at Universität Hamburg.
\end{acks}

\bibliographystyle{ACM-Reference-Format}
\bibliography{paper-lifting-factored-mdps}

\appendix
\section{Omitted Details}

\subsection{Examples for Hyperoperations in Relational Factored MDPs} \label{appendix:hyperoperations}
Definition~\ref{definition:rfMDP} requires the rewards to be represented as a sum of local reward functions.
In fact, this requirement can be lifted to a broader one: Any operation $\oplus$ can run over a collection of local reward functions, as long as this operation has a hyperoperation $\bigoplus$, which performs $\oplus$ repeatedly.
The most basic example is addition ($\oplus = +$) and multiplication as the corresponding hyperoperation ($\bigoplus= \cdot$).
Moreover, these operations can be linked together in any way: We compute each operation on a collection of local reward functions in a lifted manner and then compute the overall result following the linking.

Let us first define some local reward functions and the reward function:
\begin{example}
    Assume we have PRVs $A(M)$ and $B(M)$, both with Boolean domain.
    We define the local reward functions $r_1(A(M))$ with $r_1(false) = 2$ and $r_1(true) = 5$.
    We further define $r_2(B(M))$ with $r_2(false)=-1$ and $r_2(true) = 1$.
    We define the reward function as $\sum_{m \in M} r_1(m) \cdot \prod_{m \in M} r_2(m)$
\end{example}

The lifted computation of the reward function then calculates each basic term, i.e., $\sum_{m \in M} r_1(m)$ and $\prod_{m \in M} r_2(m)$, separately and combines the results using the linking operations, which is in this example the multiplication:
\begin{example}
    Assume $\abs{\mathcal{D}(M)}=5$ and $A(M)$ is true for two of them and $B(M)$ is true for three of them.
    The lifted computation of the first term is then:
    \begin{equation}
        \sum_{m \in M} r_1(m) = 2 \cdot r_1(true) + 3 \cdot r_2(false),
    \end{equation}
    and for the second term:
    \begin{equation}
        \prod_{m \in M} r_2(M) = \left(r_2(true)\right)^3 \cdot \left(r_2(false)\right)^2.
    \end{equation}
    The overall result is thus:
    \begin{equation}
        \left( 2 \cdot r_1(true) + 3 \cdot r_2(false) \right) \cdot \left( \left(r_2(true)\right)^3 \cdot \left(r_2(false)\right)^2 \right).
    \end{equation}
\end{example}

The hyperoperation is needed to \emph{reuse} the result of one representative, lifting it to the group.

\subsection{Formal Definition of Relational Cost Graph} \label{appendix:rcg_formal_definition}
To give a formal representation of the relational cost graph, we first define an undirected graph:

\begin{definition}[Graph]
    A graph $G = (V,E)$ is tuple of a set $V$ of vertices and $E \subseteq V \times V$ of edges.
    We require $E$ to be symmetric, i.e., $(a,b) \in E \implies (b,a) \in E$.
\end{definition}

We further need some shorthand syntax to extract the logvars a PRV is defined over and to extract the PRVs a parfactor or local reward function is defined over:

\begin{definition}[Argument Extraction]
    Given a PRV $B$ defined over logvars $L_1,\dots,L_n$, the notation $args(B)$ returns the logvars $L_1,\dots,L_n$, i.e.,
    \begin{equation}
        args(B) \coloneqq \{L_1,\dots,L_n\}.
    \end{equation}
    We analogously define $args$ to return the arguments of a parfactor or local reward function:
    Given a parfactor $\phi$ defined over PRVs $B_1, \dots B_m$, $args(\phi)$ returns those PRVs:
    \begin{equation}
        args(\phi) \coloneqq \{B_1, \dots, B_m\}.
    \end{equation}
    And given a local reward function $r$ defined over PRVs $A_1, \dots, A_k$, $args(r)$ returns those PRVs:
    \begin{equation}
        args(r) \coloneqq \{A_1,\dots,A_k\}.
    \end{equation}
\end{definition}

Using these shorthand notations, we can give a formal definition of the relational cost graph:
\begin{definition}[Formal Definition of Relational Cost Graph]
    Given an rfMDP with a set $\mathbf{R}$ of parameterized local reward functions, a set $\mathbf{\Phi}$ of factors, and a set $\mathbf{X}$ of PRVs, the \emph{relational cost graph} is a graph $rcg = (V,E)$ with $V \coloneqq \mathbf{X}$ and the set $E$ of edges given by
    \begin{align}
        E \coloneqq \{ (A,B) &\in \mathbf{X} \times \mathbf{X} \mid \left( args(A) \cap args(B) \neq \emptyset \right) \\
        &\land \left( \exists \phi \in \mathbf{\Phi}: a \in args(\phi) \land b \in args(\phi) \right) \\
        &\land \left( \exists r \in \mathbf{R}: a \in args(r) \land b \in args(r) \right)\}.
    \end{align}
\end{definition}

\subsection{Examples for the Relational Cost Graph} \label{appendix:rcg-examples}
We illustrate the relational cost graph and later the state representation on a bit more complex setting:

\begin{example} \label{ex:relational_cost_graph}
Consider the PRVs $Sick(M)$ and $RemoteWork(M)$ and assume we have a parfactor defined over these two PRVs as well as the PRV $Sick'(M)$ for the next state.
Then, the relational cost graph consists of two vertices $Sick(M)$ and $RemoteWork(M)$ with an edge between them.
Figure~\ref{fig:sick_rewo_rcg} shows the graph visually.
Furthermore, we can understand why these two PRVs need to be counted together: A person is (not) sick in the next state dependent on that person being (not) sick and (not) working remote in the current state.
We need both values for the correct transition probability.
\end{example}

\begin{figure}[h]
    \centering
    \scalebox{0.7}{
    \begin{tikzpicture}[baseline=(current bounding box.center),
                        nodesize/.style={minimum size=1.6cm}]
        \node[node, nodesize] (s) {$Sick(M)$};
        \node[node, nodesize] (r) [right= of s] {$ReWo(M)$};
    
        \draw[-] (s) -- (r);
    \end{tikzpicture}
    }
    \caption{Visual representation of the relational cost graph in Example~\ref{ex:relational_cost_graph}. We abbreviate $RemoteWork(M)$ by $ReWo(M)$.}
    \label{fig:sick_rewo_rcg}
    \Description{The relational cost graph consists of two vertices Sick(M) and RemoteWork(M) with an edge between them.}
\end{figure}

We now apply the definition of a CRV to Example~\ref{ex:relational_cost_graph}.
\begin{example}
The CRV for the clique consisting of $Sick(M)$ and $RemoteWork(M)$ has the structure $\#[Sick(M),RemoteWork(M)]$ and a histogram for that CRV would consist of the four entries $n_{tt}$ for the number of people being sick and working remote, $n_{tf}$ for the number of people sick but not working remote, $n_{ft}$ for the number of people not sick but working remote, and $n_{ff}$ for the number of people not sick and not working remote.
\end{example}

We also provide an example for using our CRVs for counting PRVs defined over different parameters, where regular Counting Random Variables cannot be used:
\begin{example}
Suppose we have a parfactor defined over PRVs $Sick(X),Friends(X,Y),Sick(Y)$. The respective CRV has structure $\#[Sick(X),Friends(X,Y),Sick(Y)]$ and we count for each (joint) assignment to the groundings of the PRVs how often this assignment occurs.
Assume that $X$ has a domain with only one person $x$ and $Y$ with three, $y_1$ to $y_3$.
Now further assume that $x$ and $y_1$ are sick and $x$ is friends with everyone except $y_2$.
In other words, $x$ is friends with one sick and one healthy person and not friends with a sick person.
Thus, the entries in the histogram for $ttt$, $ttf$ and $tft$ are one, while all other ones are zero.
\end{example}

\subsection{Proof of Correct State Representation} \label{appendix:state-representation}

\begin{theorem*}
The representation in Definition~\ref{def:state-representation} is correct.
\end{theorem*}

\begin{proof}
We first lay the foundation for the proof and then continue to show both directions of transformation.

By the definition of a parfactor model $G=\{g_i\}$, the full joint probability distribution is
\begin{align}
    P_G &= \frac{1}{Z} \prod_{f \in gr(G)} f \\
    &= \frac{1}{Z} \prod_{g \in G} \prod_{f \in gr(g)} f,
\end{align}
with $gr(g)$ referring to the groundings of a parfactor $g$.
Without loss of generality, we fix a parfactor $g$.
We split the PRVs $\mathcal{B}$ the factor $g$ is defined over into two sets:
The set $\mathbf{B}_{in}$ of \emph{input} PRVs, which contribute to the current state, and the set $\mathbf{B}_{out}$ of \emph{output} PRVs, which contribute to the next state.
We can ignore the set $\mathbf{B}_{out}$, because we only represent the current state and iterate over the next state later in Foreplan.

Given groundings for $g$, we count the occurrences of each assignment to the PRVs in $\mathbf{B}_{in}$.
We split the counts into different CRVs according to Definition~\ref{def:state-representation}.

Given a state in CRVs, we have to instantiate all possible groundings respecting the current state representation.
The PRVs in $\mathbf{B}_{in}$ are possibly split across different CRVs.
We can combine any grounding of the PRVs in the different sets, as these PRVs do not share a parameter.
\end{proof}

\subsection{Calculating Transition Probabilities} \label{appendix:transition_probabilities}
In this section, we show how Foreplan calculates the transition probabilities for the constraints in the linear program in Equation~\ref{eq:lp}.
For each state and action combination, Foreplan generates one constraint.
Within this constraint, a sum is taken over all future states.
We show how to calculate the required $P(s' \mid s,a)$ for given state $s$, action $a$ and next state $s'$.
We assume that we have one parfactor per PRV in the next state.

Since the transition function is factored, Foreplan calculates the probability of each PRV in the next state separately.
Thus, we only need to describe how Foreplan calculates $P(s_j' \mid s,a)$ for a given current state, given action and given parfactor for state variable $s_j'$ in the next state.
In short, Foreplan first computes the state representation for the current and next state zoomed in at the given parfactor and second iterates over all possible transitions, summing their probabilities.
We describe the two steps in more detail.

First, Foreplan computes the state representation only for the input PRVs (including the action PRV) of the parfactor from the given state representation and for the output PRV.
Depending on the intertwinedness of the parfactors, the computation is just an extraction or summing out unneeded PRVs from the state representation.
At the end of this step, Foreplan has a CRV for the input PRVs and another one for the output CRV for the parfactor.

Let us fix an order the buckets of the input CRV and denote the counts in each bucket by $k_i, i = 1, \dots, n$, where $n$ is the number of buckets.

In the second step, Foreplan iterates over all possible transitions, where one transition specifies transition counts $t_i, i = 1, \dots n$.
The transition count $ 0 \leq t_i \leq k_i$ specifies the number of objects transitioning from bucket $i$ to the true-assignment of the output CRV, $k_i - t_i$ to the false-assignment, respectively.
Since the next state is given, the sum over the $t_i$ must be exactly the number of true-assignments in the output CRV.
For a fixed transition, Foreplan calculates the probability $P_{t_i,k_i}$ of a fixed bucket $i$ by
\begin{equation}
    P_{t_i,k_i} = \binom{k_i}{t_i} \cdot \phi_{i \to true}^{t_i} \cdot \phi_{i \to false}^{k_i-t_i},
\end{equation}
where $\phi_{i \to true}$ denotes the probability for transitioning from bucket $i$ to $true$ and $\phi_{i \to false}$ for the transition to $false$.
Both probabilites can be looked-up in the parfactor.
For a parfactor for the output PRV $s_j$, the computation is thus
\begin{equation} \label{eq:exact_foreplan_transition}
    P(s_j \mid s,a) = \sum_{(t_i)_i} \prod_i P_{t_i,k_i},
\end{equation}
where the sum $\sum_{(t_i)_i}$ goes over all possible transitions.

We illustrate Equation~\ref{eq:exact_foreplan_transition} with our running example:
\begin{example}
We show how to calculate the probability of the PRV $Travel$.
Let us denote the number of persons travelling in the current state by $x$, the number of persons travelling in the next state by $x'$, the number of persons travelling and restricted from travelling by $a_1$ and the number of persons not travelling and restricted from travelling by $a_2$.
Then, the sum in Equation~\ref{eq:exact_foreplan_transition} goes over all $t_1,t_2,t_3,t_4 \geq 0$ with
\begin{align}
    t_1 + t_2 + t_3 + t_4 &= x', \\
    t_1 &\leq a_1, \\
    t_2 &\leq x - a_1, \\
    t_3 &\leq a_2, \\
    t_4 &\leq m - x - a_2,
\end{align}
where $m$ is the number of persons in our example.
We denote by $\phi(travel',travel,restrict)$ the probability of a person travelling in the next state given that person is currently (not) travelling and (not) being restricted from travelling.
Then, the values $P_{t_i,k_i}$ are calculated by
\begin{align}
    P_{t_1,a_1} &= \binom{a_1}{t_1} \cdot \phi(t,t,t)^{t_1} \cdot \phi(f,t,t)^{a_1-t_1} \\
    P_{t_2,x-a_1} &= \binom{x-a_1}{t_2} \cdot \phi(t,t,f)^{t_2} \cdot \phi(f,t,f)^{x-a_1-t_2} \\
    P_{t_3,a_2} &= \binom{a_2}{t_3} \cdot \phi(t,f,t)^{t_3} \cdot \phi(f,f,t)^{a_2-t_3} \\
    P_{t_4,m - x - a_2} &= \binom{m - x - a_2}{t_4} \cdot \phi(t,f,f)^{t_4} \notag \\ &\quad\quad \cdot \phi(f,f,f)^{m - x - a_2 - t_4}.
\end{align}
\end{example}

\subsection{Example of Removing a Maximum Operator} \label{appendix:ve}
Let us briefly illustrate the step of removing the maximum operator.
The removal is a two-phase process. In a first phase, ALP \emph{eliminates} the variables and in a second phase, ALP generates the constraints for the linear program along the elimination sequence.
Suppose we have the function
\begin{equation}
F = \max_{x_1,x_2} f_1(x_1) + f_2(x_1,x_2) + f_3(x_2)
\end{equation}
in a linear program, e.g., $a \geq F$ or $a = F$, where $a \in \mathbb{R}$ or $a$ as a linear program variable.
We start with the first phase.
ALP \emph{eliminates} $x_1$ by replacing $f_1$ and $f_2$ by a new function
\begin{equation}
e_1(x_2) = \max_{x_1} f_1(x_1) + f_2(x_1,x_2).
\end{equation}
ALP eliminates $x_2$ by replacing $e_1$ and $f_3$ by a new function
\begin{equation}
e_2 = \max_{x_2} e_1(x_2) + f_3(x_2).
\end{equation}
Note that $e_2$ has an empty scope and evaluates therefore to a number.
We continue with the second phase, in which ALP translates the elimination sequence into linear program constraints:
In the linear program, ALP adds helping variables and constraints to enforce the maxima in the different terms~\cite{approx_factored_mdps}.
For each function $e$ with domain $Z$, ALP adds the variables $u_z^e$ for each assignment $z$ to $Z$. The variable $u_z^e$ is supposed to yield the value of $e(z)$. For the initial functions $f_i$, in our case $f_1$, $f_2$, $f_3$, ALP simply adds $u_z^{f_i} = f_i(z)$ to the constraints of the linear program.
Suppose we got the function $e(z) = \max_{x} e_i$ when eliminating some variable $x$. ALP then adds the constraints
\begin{equation} \label{eq:max_realization}
u_z^e \geq \sum_i u_{(z,x)}^{e_i} \ \forall z,x.
\end{equation}
For $e_2$, the generated constraints would be
\begin{equation}
u^{e_2} \geq u_{x_2}^{e_1} + u_{x_2}^{f_3}
\end{equation}
for all possible values of $x_2$.
We are interested in keeping the number of added constraints small, which is the aim of VE.

\subsection{Proof of Runtime Complexity of Approximate Foreplan} \label{appendix:approx_foreplan_runtime}
We present the full proof for the runtime complexity of Approximate Foreplan: 
\begin{theorem*}
Approximate Foreplan runs in time polynomial in the number of objects, polynomial in $c$ and exponential in the induced width of each cost network, when $w$ is bound.
\end{theorem*}
\begin{proof}
The runtime of Approximate Foreplan is the runtime of precomputing the backprojections and solving the linear program in Equation~\ref{eq:factored_lp}.

The precomputation involves to calculate the backprojections for each state and action combination.
The backprojections are computed as given in Definition~\ref{def:lifted-backprojection}, which involves another iteration over the state space.
Thus, the complexity of precomputing the backprojections is bounded by the state complexity raised to the power of three.
However, the size of the state space is, in contrast to Foreplan, no longer exponential in $c$, but only polynomial in $c$:
Because of the basis functions, the CRVs of the state representation are not iterated jointly, but independently.
For the first part of Approximate Foreplan, we are left with a runtime complexity polynomial in the state space, which is polynomial in the number of objects and $c$, and exponential in $w$.

The second part is solving the linear program in Equation~\ref{eq:factored_lp}.
Linear programs can be solved in time polynomial w.r.t the number of variables and constraints~\cite{63499}.
Therefore, we count the number of variables and constraints.
Initially, we have one variable per basis function and one constraint per action.
Then, the maximum operator in each action constraint is removed.
The removal process follows the algorithm by \customcite{approx_factored_mdps}, who have shown their results for arbitrary functions and is therefore applicable.
The removal process takes time exponential in the induced width of each cost network~\cite{approx_factored_mdps,dechter1999bucket}.
Therefore, the runtime complexity for solving the linear program is polynomial in the number of actions and exponential in the induced width of each cost network.

The overall runtime complexity is
\begin{enumerate*}[label={(\roman*)}]
\item polynomial in $c$ and the number of objects,
\item exponential in the $w$ and the induced width of each cost network.
\end{enumerate*}
With $w$ bounded, the claim holds.
\end{proof}

\subsection{More Runtime Theorems for Approximate Foreplan} \label{appendix:foreplan-theorems}
We first define the total relational cost graph:
\begin{definition}[Total Relational Cost Graph] \label{def:total_relational_cost_graph}
The \emph{total relational cost graph} for a solution environment for an rfMDP contains a vertex for each (P)RV.
Two vertices are connected by an edge if they occur together in a function or parfactor.
\end{definition}

By definition, the total relational cost graph is a supergraph for all graphs of interest for the runtime complexity:
\begin{theorem} \label{theorem:total_relational_cost_graph_subgraphs}
The following graphs are each subgraphs of the total relational cost graph:
\begin{enumerate}
\item relational cost graph
\item the cost network for each maximum constraint in Approximate Foreplan
\end{enumerate}
\end{theorem}
\begin{proof}
For both cases, we show that the total relational cost graph contains \emph{more} vertices and edges than the respective definition requires. Thus, we can remove the superfluos vertices and edges to arrive at the respective subgraph.

We start with the relational cost graph.
By Definition~\ref{def:relational_cost_graph}, the relational cost graph has a vertex for each PRV and an edge between two PRVs if they share a logvar and occur together in a parfactor, a parameterized local reward function, or a basis function.
In particular, all these vertices and edges are introduced in the total relational cost graph too, when we connect two vertices once the corresponding PRVs occur together in a parfactor. We may add more edges than for the relational cost graph as we ignore the \emph{share a logvar} condition.

We continue with the cost network for each maximum constraint.
Let us take an arbitrary maximum constraint.
The cost network consists of vertices for each variable appearing in the constraint and connects two vertices with an edge if the corresponding variables appear together in the same function.
Thus, the total relational cost graph contains these edges too, as the respective variables occur together in a function.
\end{proof}

Since the induced width is the same as the treewidth minus one, we can give some more bounds:
\begin{theorem}
If the total relational cost graph for an rfMDP has bounded treewidth, Approximate Foreplan runs in time polynomial in the number of objects of the rfMDP.
\end{theorem}
\begin{proof}

By Theorem~\ref{theorem:total_relational_cost_graph_subgraphs}, the relational cost graph is a subgraph of the total relational cost graph.
As a subgraph can have the treewidth of the supergraph at most~\cite{db122a323d0740a3b2b1580529f310e7}, the treewidth of the relational cost graph is bounded.
And if a graph has bounded treewidth, it also has bounded clique number, which is our $w$~\cite{db122a323d0740a3b2b1580529f310e7}.
Thus, Theorem~\ref{theorem:approximate:relational_cost_network_bounded_clique} follows, leaving only the induced width of each cost network.

By Theorem~\ref{theorem:total_relational_cost_graph_subgraphs}, each cost network is a subgraph of the total relational cost graph.
The treewidth of each cost network is bounded by the treewidth of the total relational cost graph.
As the induced width equals the treewidth minus one, it is bound, leaving no variable with exponential influence.
\end{proof}

\subsection{Proof of Approximation Guarantee} \label{appendix:foreplan-approximation}

We now provide the proof that Approximate Foreplan and ALP yield the same results on rfMDPs:
\begin{theorem*}
Given  an rfMDP $R$, Approximate Foreplan and ALP are equivalent on $R$ and the grounded version of $R$, respectively.
\end{theorem*}
\begin{proof}
We first prove that the basis functions and backprojections evaluate to the same terms. Then, we investigate the setup of the linear programs.

The lifted basis functions accumulate multiple grounded ones.
Evaluating a lifted basis function yields the same result as summing the grounded basis functions. 
For the backprojections, the case is very similar.
Evaluating the lifted backprojections as in Definition~\ref{def:lifted-backprojection} yields the same result as summing all grounded backprojections, because the objects grouped through PRVs are indistinguishable and share the same transition probabilities. 
Since we sum over all backprojections and basis functions in the linear program, the equivalence of the sums suffices.

For the linear program, we start with the objective function and continue with the constraints.
The objective function used in Approximate Foreplan groups multiple grounded basis functions, which are used in ALP, together in one lifted basis function. Therefore, if we denote the weights in ALP by $\alpha'_i$, we have the weights $\alpha_i = n_i \cdot \alpha'_i$ for Foreplan, where $n_i$ stands for the number of grouped basis functions.
Since the grouped basis functions are indistinguishable, the weights used in Approximate Foreplan are evenly distributed in ALP.
Next, we have the constraints.
Since the backprojections and basis functions in Approximate Foreplan evaluate to the same terms as the grounded functions in ALP, each individual constraint is correct.
Furthermore, by Theorem~\ref{theorem:state-representation-correctness}, Approximate Foreplan covers the whole action space.
\end{proof}

\subsection{Evaluation} \label{appendix:evaluation}

Figure~\ref{fig:sysadmin_total_runtime_appendix} shows the runtime of (Approximate) Foreplan, ALP and XADD Symbolic Value Iteration (VI) on the fully-connected SysAdmin~\cite{approx_factored_mdps} example with a time limit of two hours.
Both versions of Foreplan clearly surpass ALP and XADD Symbolic VI.

For the BoxWorld~\cite{boutilier2001symbolic} example, we use trains instead of trucks and increase the number of boxes and trains evenly and simultaneously.
Figure~\ref{fig:boxworld_total_runtime_appendix} shows the runtime of (Approximate) Foreplan, ALP and XADD Symbolic VI on the BoxWorld example with a time limit of 15 hours.
Foreplan is more than $83.227$ times faster than the exact XADD Symbolic VI, which times out after two boxes and trains.
Foreplan is at nine boxes and trains even more than three times faster than ALP.
For approximation, Approximate Foreplan is more than $2.057$ times faster than ALP at nine boxes and trains.
Moreover, Approximate Foreplan manages to go up to 30 boxes and trains within the same time limit.

\begin{figure}[tb]
	\centering
	\includegraphics[width=\linewidth]{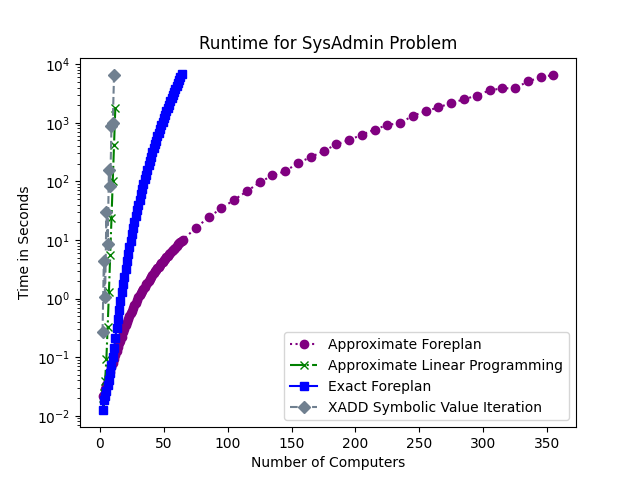}
	\caption{Runtime (logscale) of (Approximate) Foreplan, ALP and XADD Symbolic VI on the SysAdmin example with a time limit of two hours. Only runs within these limits are shown. For Approximate Foreplan, we have chosen a bigger step size when no other algorithm was left for the number of computers.}
	\label{fig:sysadmin_total_runtime_appendix}
    
    \Description{The figure shows the runtime results of Foreplan, Approximate Foreplan, ALP and XADD Symbolic VI on the SysAdmin example.}
\end{figure}

\begin{figure}[tb]
	\centering
	\includegraphics[width=\linewidth]{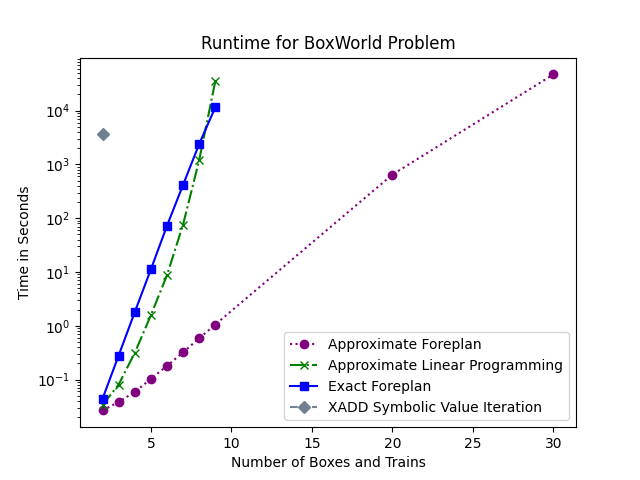}
	\caption{Runtime (logscale) of (Approximate) Foreplan, ALP and XADD Symbolic VI on the BoxWorld example with a time limit of 15 hours. Only runs within these limits are shown. For Approximate Foreplan, we have chosen a bigger step size when no other algorithm was left for the number of boxes and trains.}
	\label{fig:boxworld_total_runtime_appendix}
    \Description{The figure shows the runtime results of Foreplan, Approximate Foreplan, ALP and XADD Symbolic VI on the BoxWorld example.}
\end{figure}

The approximation error in the BoxWorld example is more complex.
With our chosen basis functions, ALP and Approximate Foreplan both deviate in up to $85$ \% of all actions, when tested on up to eight boxes and trains.
Still, Approximate Foreplan is not worse than ALP, since both report the same coefficients for the basis functions, empirically and theoretically by Theorem~\ref{theorem:equivalence_alp_aforeplan}.
The reason lies in the reward and basis functions:
The reward functions award points only for boxes in cities.
Because we have chosen the basis functions identical to the reward functions and an additional constant basis function, the basis functions do not take boxes on trains into account, where the boxes may be unloaded in a more profitable city.
Therefore, with better basis functions, we should be able to capture better solutions.
Moreover, we highlight that Approximate Foreplan returns the same policy as ALP and is more than $30.000$ times faster.

\section{Walkthrough of Approximate Foreplan} \label{apx:walkthrough}
We describe in this section how we can solve Example~\ref{ex:epidemic} with Approximate Foreplan.
We first model the example formally and find the state representation.
Afterwards, we precompute the backprojections and instantiate the linear program.

\subsection{Modeling the Small Town}
We model the setting in Example~\ref{ex:epidemic}.
The town population consists of three people.
We refer to Figure~\ref{fig:epidemic_lifted} for the transition model.
We define the parfactors $f_1$ and $f_2$ according to Tables~\ref{apx:table:epidemic_travel} and~\ref{apx:table:epidemic_sick}.
The parfactor $f_3$ sets the probability of an epidemic in the next state to $\frac{s+1}{5}$, where $s$ is the number of sick persons in the current state.
\begin{table}[bp]
	\centering
    \caption{Transition probabilities for each person giving the probability of travelling in the next timestep given whether the person is currently travelling and being restricted.}
	\label{apx:table:epidemic_travel}
	\begin{tabular}{cc|c}
		$Travel(M)$ & $Restrict(M)$ & $Travel'(M)$ \\ 
		\hline
		0 & 0 & 0.2 \\
		0 & 1 & 0.1 \\
		1 & 0 & 0.9 \\
		1 & 1 & 0.5 \\
	\end{tabular}
\end{table}
\begin{table}[bp]
	\centering
    \caption{Transition probabilities for each person giving the probability of being sick in the next timestep given whether the person is currently sick and there is an epidemic.}
	\label{apx:table:epidemic_sick}
	\begin{tabular}{cc|c}
		$Sick(M)$ & $Epidemic$ & $Sick'(M)$ \\ 
		\hline
		0 & 0 & 0.2 \\
		0 & 1 & 0.8 \\
		1 & 0 & 0.4 \\
		1 & 1 & 0.6 \\
	\end{tabular}
\end{table}
The mayor can restrict an arbitrary subset of the towns population from travelling.
The local reward functions are $R_1(Sick(M))$ and $R_2(Travel(M))$, each for each person of the towns population.
The function $R_1(Sick(M))$ evaluates to $1$ if the person is not sick and to $-1$ if the person is sick (c.f. Example~\ref{ex:lifted_reward}):
\begin{equation}
R_1(Sick(M)) = \begin{cases}
-1, & Sick(M) \\
1, & \lnot Sick(M)
\end{cases}
\end{equation}
The function $R_2(Travel(M))$ evaluates to $2$ if the person is travelling and to $0$ otherwise:
\begin{equation}
R_2(Travel(M)) = \begin{cases}
2, & Travel(M) \\
0, & \lnot Travel(M)
\end{cases}
\end{equation}
We set the discount factor $\gamma = 0.9$.
For choosing the basis functions, we go with the ones in Example~\ref{ex:lifted_basis_function}, that is, $h_0 = 1$, $h_1 = R_1$ and $h_2 = R_2$.

\subsection{State Space Representation}
For using Approximate Foreplan, we first build the relational cost graph to find the state space representation.
Afterwards, we can modify the actions to be compatible with the state space representation.
The relational cost graph, depicted in Figure~\ref{fig:epidemic_rcg}, contains two vertices, $Sick(M)$ and $Travel(M)$, because these are the only PRVs in the example.
The relational cost graph does not contain any edge, because the two PRVs do not occur together in a parfactor, reward or basis function.
Thus, the vertices $Sick(M)$ and $Travel(M)$ each are a clique of size one.
Therefore, the state space representation contains two histograms along the value of the propositional random variable $Epidemic$. More formally, the state space representation is $([Sick(M)],[Travel(M)],Epidemic)$ (c.f. Example~\ref{ex:state_representation}).

For the action $Restrict(M)$, the mayor no longer needs to specify on which person(s) she applies the travel ban(s) on, but has to define how many persons currently (not) travelling receive a travel restriction. Thus, the mayor needs to define the action histogram $[Travel(M),Restrict(M)]$ giving the numbers of people (not) travelling that are (not) restricted from travelling (c.f. Example~\ref{ex:action_with_state}).

\subsection{Computing Backprojections} \label{apx:walkthrough-backprojections}
The first step of Foreplan is the precomputation of the backprojections.
The backprojections are generally defined as
\begin{equation}
g_i^a(x) = \sum_{x'} P(x' \mid x,a) \cdot h_i(x'),
\end{equation}
where $x'$ are the parameters of $h_i$, $x$ are the parents of $x$ in the transition model and $a$ is the selected action.
For abbreviation, we use $S(M)$ for $Sick(M)$, $R(M)$ for $Restrict(M)$, $T(M)$ for $Travel(M)$ and $Epi$ for $Epidemic$ in this subsection.
We start with the backprojection of $h_0 = 1$:
\begin{equation}
g_0^a(x) = \sum_{x'} P(x' \mid x,a) \cdot h_0(x) = 1 \cdot \sum_{x'} P(x' \mid x,a) = 1.
\end{equation}
Next, we backproject $h_1(Sick(M))$. Note that $Sick'(M)$ is independent of the action in the transition model. Thus, the backprojection is independent of $a$:
\begin{equation}
\begin{split}
g_1(S(M),&Epi) = \\
&\sum_{x' \in \{t,f\}} P(S'(M) = x' \mid S(M), epi) \cdot h_1(x')
\end{split}
\end{equation}
We plug the probabilities from Table~\ref{apx:table:epidemic_sick} in and receive
\begin{alignat}{4}
&g_1(true, true) &&= 0.6 \cdot (-1) + 0.4 \cdot 1 &&= -&&0.2, \\
&g_1(true, false) &&= 0.4 \cdot (-1) + 0.6 \cdot 1 &&= &&0.2, \\
&g_1(false, true) &&= 0.8 \cdot (-1) + 0.2 \cdot 1 &&= -&&0.6, \\
&g_1(false, false) &&= 0.2 \cdot (-1) + 0.8 \cdot 1 &&= &&0.6.
\end{alignat}
The backprojection of $h_2$ is a bit more complex, as it includes the action.
The general backprojection of $h_2$ is given by
\begin{equation}
\begin{split}
g_2^{R(M)}&(T(M)) =\\
&\sum_{x' \in \{t,f\}} P(T'(M) = x' \mid T(M),R(M)) \cdot h_2(x').
\end{split}
\end{equation}
We distinguish the two cases $R(M)=false$ and $R(M)=true$.
We start with the first one, which leads to:
\begin{equation}
\begin{aligned}
g_2^{f}(t) &= \sum_{x' \in \{t,f\}} P(T'(M)=x' \mid T(M)=t,R(M)=f) \\
&= 0.9 \cdot h_2(t) + 0.1 \cdot h_2(f) = 0.9 \cdot 2 \\
&= 1.8
\end{aligned}
\end{equation}
and
\begin{equation}
\begin{aligned}
g_2^{f}(f) &= \sum_{x' \in \{t,f\}} P(t'(P)=x' \mid T(M)=f,R(M)=f) \\
&= 0.2 \cdot h_2(t) + 0.8 \cdot h_2(f) = 0.2 \cdot 2 \\
&= 0.4.
\end{aligned}
\end{equation}
We continue with $R(M)=true$, leading to
\begin{equation}
\begin{aligned}
g_2^{t}(t) &= \sum_{x' \in \{t,f\}} P(T'(M)=x' \mid T(M)=t,R(M)=t) \\
&= 0.5 \cdot h_2(t) + 0.5 \cdot h_2(f) = 0.5 \cdot 2 \\
&= 1
\end{aligned}
\end{equation}
and
\begin{equation}
\begin{aligned}
g_2^{t}(f) &= \sum_{x' \in \{t,f\}} P(T'(M)=x' \mid T(M)=f,R(M)=t) \\
&= 0.1 \cdot h_2(t) + 0.9 \cdot h_2(f) = 0.1 \cdot 2 \\
&= 0.2.
\end{aligned}
\end{equation}

\subsection{Instantiation of Linear Program} \label{apx:walkthrough-lp_instantiation}
The second step of Approximate Foreplan is to instantiate the linear program.
The linear program has three variables $w_1$, $w_2$ and $w_3$.
The objective function is to minimize $\sum_{i=1}^3 \alpha_i w_i$, with $\alpha_i$ being the relevance weights.
We have one constraint for each possible action $a$:
\begin{equation} \label{apx:lp:constraint}
\begin{aligned}
0 \geq &\max_{x_1,x_2,x_3} \{ 3 - 2 \cdot x_1 + 2 \cdot x_2 + w_0 \cdot (0.9 \cdot 1 - 1) \\
&+ w_1 \cdot (0.9 \cdot G_1(x_1,x_3) - H_1(x_1)) \\
&+ w_2 \cdot (0.9 \cdot G_2^a(x_2) - H_2(x_2)) \},
\end{aligned}
\end{equation}
where $x_1,x_2 \in \{0,1,2,3\}$ specify the number of persons being sick and travelling, respectively.
The variable $x_3$ is Boolean and stores the truth value of $Epidemic$.
With $G_i$ and $H_i$ we denote the lifted computation of $g_i$ and $h_i$. In the following, we write the lifted computation in terms of $x_i$ and $g_i$ or $h_i$.
In the remainder of this subsection, we show the complete constraint generation for the linear program for one example action. We omit all other actions because of limited insights compared to one example instantiation and constraint generation.
We show the constraint generation for the action of restricting nobody.
We start with writing the constraint tailored to this setting and directly computing the value of basis functions and their backprojections in a lifted way:
\begin{equation}
\begin{aligned}
0 \geq &\max_{x_1,x_2,x_3} \{3 - 2x_1 + 2x_2 - 0.1 w_0 \\
&+ w_1 \cdot (x_1 \cdot 0.9 \cdot g_1(t,x_3) - x_1 \cdot h_1(t)) \\
&+ w_1 \cdot ((3-x_1) \cdot 0.9 \cdot g_1(f,x_3) - (3-x_1) \cdot h_1(f)) \\
&+ w_2 \cdot (x_2 \cdot 0.9 \cdot g_2^f(t) - x_2 \cdot h_2(t)) \\
&+ w_2 \cdot ((3-x_2) \cdot 0.9 \cdot g_2^f(f) - (3-x_2) \cdot h_2(f))\}.
\end{aligned}
\end{equation}
We continue with introducing functions to save steps later in variable elimination and constraint generation. We define four functions with different parameters, each collecting the respective terms:
\begin{equation}
0 \geq \max_{x_1,x_2,x_3} \{f_1(x_1) + f_2(x_2) + f_3(x_1,x_3) + f_4\},
\end{equation}
with
\begin{align}
f_1(x_1) &= 3 - 2x_1 - w_1x_1h_1(t) - w_1 \cdot (3-x_1) \cdot h_1(f) \\
f_2(x_2) &= 2x_2 + w_2 \cdot (x_2 \cdot 0.9 \cdot g_2^f(t) - x_2 h_2(t)) \\
&\quad\quad + w_2 \cdot ((3-x_2) \cdot 0.9 \cdot g_2^f(f)) \notag\\
f_3(x_1,x_3) &= 0.9w_1 \cdot (x_1g_1(t,x_3) + (3-x_1) \cdot g_1(f,x_3)) \\
f_4 &= -0.1 w_0
\end{align}

The only task we are left with is to remove the maximum operator.
For that, we describe the two phases, variable elimination and constraint generation, in two different subsections.

\subsubsection{Variable Elimination}
We first eliminate $x_2$, leading to
\begin{equation}
0 \geq \max_{x_1,x_3} \{f_1(x_1) + e_1 + f_3(x_1,x_3) + f_4 \},
\end{equation}
with
\begin{equation}
e_1 = \max_{x_2} f_2(x_2).
\end{equation}
Next, we eliminate $x_1$:
\begin{equation}
	0 \geq \max_{x_3} \{e_2(x_3) + e_1 + f_4 \},
\end{equation}
with
\begin{equation}
e_2(x_3) = \max_{x_1} f_1(x_1) + f_3(x_1,x_3).
\end{equation}
Last, we eliminate $x_3$:
\begin{equation}
	0 \geq e_1 + e_3 + f_4,
\end{equation}
with
\begin{equation}
e_3 = \max_{x_3} e_2(x_3).
\end{equation}

\subsubsection{Constraint Generation} \label{apx:lp:constraint-generation}
We generate the constraints along the elimination order of the previous subsection.
All the constraints listed in this subsection together replace the single constraint in Equation~\ref{apx:lp:constraint} for our example action of restricting nobody.
We start with the constraints for the functions $f_i$ containing $w_i$s and continue with the functions $e_i$ obtained by eliminating variables:

\paragraph{$f_1(x_1)$}
\begin{align}
u_0^{f_1} &= 3 - 3 w_1 \\
u_1^{f_1} &= 3 - 2 + w_1 - 2w_1 = 1 - w_1 \\
u_2^{f_1} &= 3 - 4 + 2w_1 - w_1 = -1 + w_1 \\
u_3^{f_1} &= 3 - 6 + 3w_1 = -3 + 3 w_1
\end{align}

\paragraph{$f_2(x_2)$}
\begin{align}
u_0^{f_2} &= w_2 \cdot (3 \cdot 0.9 \cdot 0.4) = 3.1 w_2 \\
u_1^{f_2} &= 2 + w_2 \cdot (1 \cdot 0.9 \cdot 1.8 - 1.2) + w_2 \cdot (2 \cdot 0.9 \cdot 0.4) \notag \\
	&= 2 + 0.34 w_2 \\
u_2^{f_2} &= 4 + w_2 \cdot (2 \cdot 0.9 \cdot 1.8 - 2 \cdot 2) + w_2 \cdot (1 \cdot 1.9 \cdot 0.4) \notag \\
	&= 4 - 0.04 w_2 \\
u_3^{f_2} &= 6 + w_2 \cdot (3 - 0.9 \cdot 1.8 - 3 \cdot 2) = 6 - 1.14 w_2
\end{align}

\paragraph{$f_3(x_1,x_3)$}
\begin{align}
u_{00}^{f_3} &= w_1 \cdot 3 \cdot 0.9 \cdot 0.6 = 1.62 w_1\\
u_{01}^{f_3} &= -1.62 w_1\\
u_{10}^{f_3} &= w_1 \cdot 1 \cdot 0.9 \cdot 0.2 + w_1 \cdot 2 \cdot 0.9 \cdot 0.6 = 1.26 w_1\\
u_{11}^{f_3} &= -1.26 w_1\\
u_{20}^{f_3} &= w_1 \cdot 2 \cdot 0.9 \cdot 0.2 + w_1 \cdot 1 \cdot 0.9 \cdot 0.6 = 0.9 w_1\\
u_{21}^{f_3} &= -0.9 w_1\\
u_{30}^{f_3} &= w_1 \cdot 3 \cdot 0.9 \cdot 0.2 = 0.54 w_1\\
u_{31}^{f_3} &= -0.54 w_1
\end{align}

\paragraph{$f_4$}
\begin{equation}
u^{f_4} = -0.1 w_0
\end{equation}

\paragraph{$e_1$}
\begin{align}
u^{e_1} &\geq u_0^{f_2} \\
u^{e_1} &\geq u_1^{f_2} \\
u^{e_1} &\geq u_2^{f_2} \\
u^{e_1} &\geq u_3^{f_2}
\end{align}

\paragraph{$e_2(x_3)$}
\begin{align}
u_0^{e_2} &\geq u_0^{f_1} + u_{00}^{f_3} \\
u_0^{e_2} &\geq u_1^{f_1} + u_{10}^{f_3} \\
u_0^{e_2} &\geq u_2^{f_1} + u_{20}^{f_3} \\
u_0^{e_2} &\geq u_3^{f_1} + u_{30}^{f_3} \\
u_1^{e_2} &\geq u_0^{f_1} + u_{01}^{f_3} \\
u_1^{e_2} &\geq u_1^{f_1} + u_{11}^{f_3} \\
u_1^{e_2} &\geq u_2^{f_1} + u_{21}^{f_3} \\
u_1^{e_2} &\geq u_3^{f_1} + u_{31}^{f_3}
\end{align}

\paragraph{$e_3$}
\begin{align}
u^{e_3} &\geq u_0^{e_2} \\
u^{e_3} &\geq u_1^{e_2}
\end{align}

\paragraph{Final Constraint}
In the end, we add the final constraint
\begin{equation}
0 \geq u^{e_1} + u^{e_3} + u^{f_4}
\end{equation}
due to the introduction of our helper functions $f_i$.

\subsubsection{Solving the Linear Program}
After having generated all constraints for all actions, the linear program is fed into a solver to compute a solution. The template for the linear program looks like this:
\begin{equation} \label{apx:lp_complete}
\begin{array}{ll@{}}
\text{Variables:} & w_1,w_2,w_3 \ ;\\
\text{Minimize:} & \alpha_1 w_1 + \alpha_2 w_2 + \alpha_3 + w_3 \ ;\\
\text{Subject to:} & \forall a \in A: \\
& \text{set of constraints generated like in Section~\ref{apx:lp:constraint-generation}}. \\
\end{array}
\end{equation}

\end{document}